 \documentclass[final,5p,times,twocolumn,authoryear]{elsarticle}

\usepackage{framed,multirow}
\usepackage{amssymb}
\usepackage{latexsym}
\usepackage{url}
\usepackage{xcolor}
\definecolor{newcolor}{rgb}{.8,.349,.1}
\usepackage{bm}
\usepackage{amsthm}
\newtheorem{theorem}{Theorem}
\newtheorem{lemma}{Lemma}
\newtheorem{definition}{Definition}%

\usepackage{multicol}
\usepackage{multirow}
\usepackage{booktabs}
\usepackage{color}
\usepackage{rotating}
\usepackage{amsmath,amsfonts}
\usepackage{algorithmic}
\usepackage{algorithm}
\usepackage{hyperref}
\usepackage{xspace}
\makeatletter
\DeclareRobustCommand\onedot{\futurelet\@let@token\@onedot}
\def\@onedot{\ifx\@let@token.\else.\null\fi\xspace}
\def\eg{\emph{e.g}\onedot} 
\def\ie{\emph{i.e}\onedot} 
 
\def\etc{\emph{etc}\onedot}

\makeatother

\journal{Arxiv}

\begin{document}

\begin{frontmatter}



\title{Improving Model Generalization by On-manifold Adversarial Augmentation \\ in the Frequency Domain}


\author[1,*]{Chang~Liu} 
\author[2,5,*]{Wenzhao Xiang}
\author[4]{Yuan He}
\author[4]{Hui Xue}
\author[1]{Shibao Zheng}
\author[3,5]{Hang Su\corref{cor1}}

\cortext[cor1]{Corresponding author: 
  Tel.: +0-000-000-0000;}
\ead{suhangss@mail.tsinghua.edu.cn}

\address[1]{Institute of Image Communication and Networks Engineering in the Department of Electronic Engineering~(EE), Shanghai Jiao Tong University, Shanghai 200240, China.}
\address[3]{Department of Computer Science and Technology, Institute for AI, THBI Lab, Tsinghua University, Beijing 100084, China.}
\address[2]{Key Lab of Intelligent Information Processing of Chinese Academy of Sciences (CAS), Institute of Computing Technology, CAS, Beijing, 100190, China}
\address[5]{Pengcheng Laboratory, China}
\address[4]{Security Department, Alibaba Group, Beijing 100084, China.}

\begin{abstract}
Deep Neural Networks (DNNs) often experience substantial performance deterioration when the training and test data are drawn from disparate distributions. Guaranteeing model generalization for Out-Of-Distribution (OOD) data remains critical; however, current state-of-the-art (SOTA) models consistently show compromised accuracy with such data. Recent investigations have illustrated the benefits of using regular or off-manifold adversarial examples as data augmentation for improving OOD generalization. Building on these insights, we provide a theoretical validation indicating that on-manifold adversarial examples can yield superior outcomes for OOD generalization. However, generating on-manifold adversarial examples is not straightforward due to the intricate nature of real manifolds.

To overcome this challenge, we propose a unique method, Augmenting data with Adversarial examples via a Wavelet module (AdvWavAug), which is an on-manifold adversarial data augmentation strategy that is exceptionally easy to implement. In particular, we use wavelet transformation to project a clean image into the wavelet domain, leveraging the sparsity characteristic of the wavelet transformation to modify an image within the estimated data manifold. Adversarial augmentation is conducted based on the AdvProp training framework. Comprehensive experiments on diverse models and datasets, including ImageNet and its distorted versions, reveal that our method markedly bolsters model generalization, especially for OOD data. By incorporating AdvWavAug into the training process, we have attained SOTA results on recent transformer-based models.
\end{abstract}




\end{frontmatter}



\def\thefootnote{1}\footnotetext{These authors contributed equally to this work}\def\thefootnote{\arabic{footnote}}

\section{Introduction}
Although deep neural networks (DNNs) have achieved remarkable performance across various tasks, their reliance on the strong assumption of independent and identically distributed (i.i.d.) training and test samples often falls short in real-world applications. Due to biased processes of collecting training data and confounding factors~\citep{hendrycks2019benchmarking, recht2019imagenet, corbett2018measure, fonseca2022similarity}, the occurrence of out-of-distribution (OOD) samples is commonplace. From a statistical standpoint, the distribution of OOD samples may deviate from the training data due to unknown shifts, ultimately resulting in catastrophic failure of DNNs' generalization to OOD samples, as evidenced in previous work~\citep{calian2021defending}. Even minor corruptions such as blurring or noise can severely deteriorate the performance of existing classifiers~\citep{vasiljevic2016examining, geirhos2018generalisation}. Consequently, the limited generalization to OOD data has hindered the practical deployment of DNNs.

\begin{figure*}[t]
  \centering
  \includegraphics[width=0.9\linewidth]{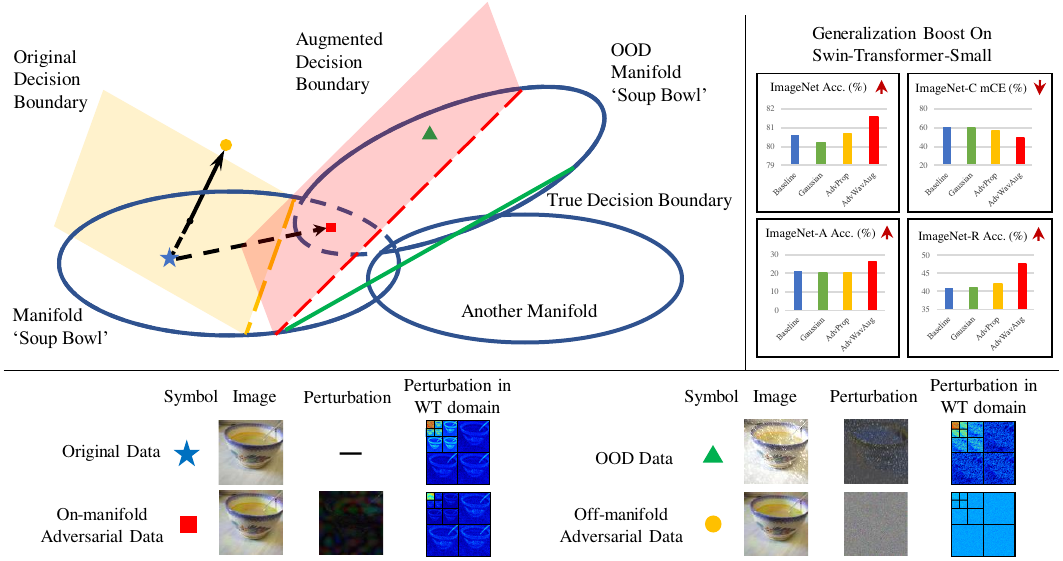}
  \caption{Boosting model generalization with on-manifold adversarial examples in the frequency domain. The OOD data (denoted as a green triangle) resides on a manifold, which is connected but distribution shifted from the original data manifold. An off-manifold adversarial example (denoted as a red point) moves outside the manifold, while an on-manifold one (denoted as a blue box) remains on the manifold. The on-manifold perturbations are more closely aligned with the semantic meaning in the wavelet domain. Extensive results, such as Swin Transformer Small~\citep{liu2021swin}, demonstrate that our AdvWavAug significantly improves model generalization on ImageNet~\citep{russakovsky2015imagenet} and its distorted versions across various backbone models.
  }
  \label{fig:1}
\end{figure*}

Data augmentations have been proven to be effective in improving model generalization to OOD data. Previous work~\citep{devries2017improved,yun2019cutmix,zhang2017mixup} has reduced generalization errors by randomly selecting transform-based augmentations. Some of these methods~\citep{cubuk2019autoaugment,cubuk2020randaugment} have even learned adaptive augmentation policies. However, these approaches heavily rely on the diversity of the selected transforms, leaving uncertainty regarding the extent to which a particular transform can truly benefit OOD generalization. A more active data augmentation is adversarial training in~\citep{tramer2017ensemble}, which augments input data with adversarial examples. Most of previous work~\citep{das2018shield,akhtar2018defense,papernot2016distillation,xie2017mitigating} has primarily focused on regular adversarial examples, such as those generated by the Projected Gradient Descent (PGD) method~\citep{madry2017towards}. Regular adversarial examples establish an upper bound on corrupted risks, as demonstrated in~\citep{yi2021improved}, indicating their potential to improve OOD generalization. However, regular adversarial examples often exhibit noise-like patterns in the background, which are unrelated to the objects present in the clean images. Consequently, augmentations with these off-manifold adversarial examples may encounter internal friction, wherein the noises that DNNs can defend against are infrequent in real-world applications. This limitation could potentially restrict their effectiveness in enhancing OOD generalization.

To leverage adversarial augmentations effectively, explore the generalization boundary, and improve OOD generalization with more certainty, we employ on-manifold adversarial examples for data augmentation. In this context, the data manifold is defined as a lower-dimensional latent space within the high-dimensional data space observed in the real world. This concept, introduced in~\citep{stutz2019disentangling}, helps separate common data from regular adversarial examples. On-manifold adversarial examples are more closely related to semantic meanings compared to off-manifold adversarial examples. Consequently, data augmentation using these on-manifold examples directs the classifier's attention toward semantic changes, which are more advantageous for OOD generalization. We provide a theoretical analysis demonstrating that models resilient to on-manifold adversarial examples have an upper bound of corrupted risks, which is commonly smaller than models resilient to off-manifold adversarial examples. It is noted that we specifically focus on the OOD data within the context of distributional robustness, rather than considering OOD data with correlation shift or diversity shift described in~\citep{ye2022ood}. However, generating on-manifold adversarial examples remains a challenging task due to the general lack of knowledge about the real manifold. Some approaches~\citep{zhou2020manifold, stutz2019disentangling} approximate the lower-dimensional data manifold by training a variational autoencoder (VAE)~\citep{kingma2013auto}.
However, representing the manifold with a VAE in large-scale datasets is inefficient, as it incurs high computational costs for training and generating additional samples. Inspired by compression algorithms (\eg,~\citep{villasenor1995wavelet, rabbani2002jpeg2000}), which indicate the high sparsity of images in the frequency domain, we aim to approximate the manifold in the frequency domain. We find that perturbing the non-sparse coefficients in the frequency domain is a similar but more efficient process, compared with the VAE implementation described in~\citep{stutz2019disentangling}.

In this paper, we propose a novel framework to enhance the generalization of DNN models by Augmenting data with Adversarial examples via a Wavelet module (AdvWavAug). Our framework incorporates a wavelet projection module, which consists of a discrete wavelet transform and its inverse, to obtain the wavelet coefficients of the input data. These coefficients are then modified based on the gradients back-propagated from the loss function. Unlike regular adversarial attacks, we introduce a multiplicative attention map for the wavelet coefficients, allowing any perturbation to be added to the attention map. The advantage of using a multiplicative attention map is that it preserves the sparsity of the wavelet coefficients, as coefficients at sparse positions will always be zeros. Finally, the modified adversarial examples are generated by applying the inverse wavelet transform to the product of the wavelet coefficients and the modified attention map. This entire process can be seen as projecting the back-propagated gradient into the wavelet domain while preserving the adversarial perturbations present at the non-sparse coefficients.

We also provide a detailed analysis that demonstrates how augmentations with on-manifold adversarial examples lead to an upper bound of corrupted risks. To verify the effectiveness of our method, we conduct model training with these on-manifold adversarial examples. Note that AdvProp~\citep{xie2020adversarial} has more advantages in improving generalization to OOD data. Therefore, we adopt it as the underlying training framework for our AdvWavAug.

Experiments on the ImageNet dataset~\citep{russakovsky2015imagenet} and its distorted versions (ImageNet-A~\citep{hendrycks2021natural}, ImageNet-R~\citep{hendrycks2021many} and ImageNet-C~\citep{hendrycks2019benchmarking}) validate the effectiveness of our proposed method. Using Swin Transformer Small~\citep{liu2021swin} as an example, compared with vanilla AdvProp, our method significantly improves generalization on ImageNet by $0.9\%$, ImageNet-A by $5.5\%$, ImageNet-R by $5.6\%$ and ImageNet-C by $8.0\%$. We have also implemented adversarial augmentation using a pre-trained VQ-VAE~\citep{razavi2019generating} network and compared its effectiveness with our method. While achieving similar performance on OOD data, our method significantly reduces the total training time by $59.4\%$ compared to the VQ-VAE implementation. Due to the assistance of on-manifold adversarial examples in understanding the semantic meaning of an image, our method further enhances the generalization of models pre-trained using self-supervised learning. For example, we obtain better results by integrating our method into the fine-tuning process of the Masked AutoEncoder (MAE) in~\citep{he2021masked}, achieving improvements on ImageNet by $0.3\%$, ImageNet-A by $1.8\%$, ImageNet-R by $1.4\%$ and ImageNet-C by $3.7\%$ on Vision Transformer Large~\citep{dosovitskiy2020image}.

In summary, we make the following technical contributions:
\begin{itemize}
    \item To solve the OOD generalization problem, we establish connections between on-manifold adversarial examples and OOD samples, and transfer the original problem into improving robustness to on-manifold adversarial examples. Detailed proof has been provided to verify that models robust to on-manifold adversarial examples have an upper-bounded generalization error to OOD samples.
    \item We propose AdvWavAug, an adversarial augmentation module, to approximate the data manifold in the frequency domain. This module projects noise-like perturbations into the frequency domain and modifies the components located on the non-sparse coefficients. We integrate AdvWavAug into the AdvProp framework to evaluate the effectiveness of our method in terms of model generalization.
    \item We conduct comprehensive experiments on ImageNet and its wide range of variants on OOD data. The results demonstrate that our algorithm significantly enhances the generalization of modern DNN architectures. We have achieved new SOTA results on two transformer-based architectures, including Vision Transformers and Swin Transformers.
\end{itemize}

\section{Background}
In this section, we will provide a brief introduction of augmentation techniques, encompassing both normal transform-based augmentations and adversary-based approaches.

\subsection {Transform-based Augmentations}
Data augmentation has been shown to be effective in improving model generalization. Traditional approaches often involve randomly selecting basic transforms. In image classification tasks, common transformations include random rotations, crops, and flips as described in~\citep{he2016deep}. More complex operations have been used to improve accuracy on clean data, such as Cutout~\citep{devries2017improved}, which randomly occludes parts of an input image; CutMix~\citep{yun2019cutmix}, which replaces a part of the target image with a different image; Mixup~\citep{zhang2017mixup}, which produces a convex combination of two images; DeepAugment~\citep{hendrycks2021many}, which passes a clean image through an image-to-image network and introduces several perturbations during the forward pass. To automate the combination of simple augmentation techniques, an enhanced mixup method~\citep{guo2019mixup} is proposed by adaptively adjusting the mixing policy. Other methods, like AutoAugment~\citep{cubuk2019autoaugment} and RandAugment~\citep{cubuk2020randaugment}, aim to learn augmentation policies automatically. Although some researches have focused on generalization to OOD data, these methods often struggle with unseen categories of corruptions~\citep{hendrycks2019benchmarking}. 

Besides, some methods have explored the robustness to various corruptions in the frequency domain. In~\citep{yin2019fourier}, the authors explore model robustness and augmentation effects from a Fourier perspective. Other methods have also investigated data augmentation in the frequency domain, such as RoHL~\citep{saikia2021improving}, which introduces a mixture of two expert models that specialize in high and low-frequency robustness; APR~\citep{chen2021amplitude} re-combines the phase spectrum of the current image with the amplitude spectrum of a distractor image to generate a new training sample, aiming to encourage the model to capture more structured information.

\subsection{Adversary-based Augmentations}
Adversarial training is another augmentation policy, which takes adversarial examples as the augmented data. 

\noindent \textbf{Adversarial Attack.}
Let $\bm{x}$ represent a clean image and $y$ denote the ground-truth label. A classifier can be denoted as $c(\bm{x}): \mathcal{X} \rightarrow \mathcal{Y}$, where $\bm{x} \in \mathcal{X} \subseteq \mathbb{R}^{d_0}$ is the input image and $\mathcal{Y} = \{1, 2, \cdots, N\}$ is the set of class labels. Adversaries aim to find adversarial examples $\bm{x}^{adv}$ to deceive the classifier. Formally, $\bm{x}^{adv}$ is defined in an $L_p$-norm ball centered around $\bm{x}$, which is expressed as $\|\bm{x}^{adv}-\bm{x} \|_p \leq \epsilon$ with $\epsilon$ being the maximum perturbation constraint for regular adversarial examples. Let $\mathcal{L}$ be the loss function and $\bm{\delta}$ denote the perturbation. For untargeted attacks, the objective is to maximize $\mathcal{L}(\bm{\theta},\bm{x}+\bm{\delta}, y)$ as
\begin{equation}
\label{eq:1}
\mathop{\arg \max} \limits_{\|\bm{\delta}\|_p \leq \epsilon} \mathcal{L}(\bm{\theta},\bm{x}+\bm{\delta}, y).
\end{equation}
In the white-box scenario, where the network structure, parameters, and gradients are accessible, adversarial examples can be generated by adding perturbations to clean images based on the gradient. FGSM~\citep{goodfellow2014explaining} is a common one-step gradient-based white-box attack algorithm. PGD~\citep{madry2017towards} improves FGSM by iterative optimization. Deepfool~\citep{moosavi2016deepfool} iteratively attacks input images based on the decision hyper-plane. Carlini \& Wagner's method (C\&W)~\citep{carlini2017towards} optimizes the attack problem in Lagrangian form and adopts Adam~\citep{kingma2014adam} as the optimizer. Due to the great threat brought by adversarial examples, adversarial training is proposed to defend against these malicious modifications.

\noindent \textbf{Adversarial Training.}
In the vanilla training setting for a DNN model, the objective function is defined as 
\begin{equation}
\label{eq:2}
    \mathop{\arg \min} \limits_{\bm{\theta}}\mathbb{E}_{(\bm{x},y) \sim \mathbb{D}}\mathcal{L}(\bm{\theta},\bm{x},y),
\end{equation}
where $\mathbb{D}$ represents the real data distribution, $\mathcal{L}(\cdot,\cdot,\cdot)$ is the loss function, $\bm{\theta}$ denotes the network parameters, $\bm{x}$ is the input image, and $y$ is the ground-truth label. Adversarial training enhances vanilla training by incorporating adversarially perturbed samples to improve adversarial robustness~\citep{madry2017towards}. AdversarialAugment~\citep{calian2021defending} generates adversarially corrupted examples for augmentation during model training. AdvProp~\citep{xie2020adversarial} extends the framework by considering the performance on both clean and adversarial examples, expressed as
\begin{equation} 
\label{eq:3}
\mathop{\arg \min} \limits_{\bm{\theta}} \mathbb{E}_{(\bm{x},y) \sim \mathbb{D}}(\mathcal{L}(\bm{\theta},\bm{x},y)
+\mathop{\arg \max} \limits_{\| \bm{\delta} \|_p \leq \epsilon}\mathcal{L}(\bm{\theta},\bm{x}+\bm{\delta},y)),
\end{equation}
in which $\bm{\delta}$ is the perturbation added to the clean image and $\epsilon$ is the constraint on the $L_p$-norm of $\bm{\delta}$.

However, most adversarial augmentation techniques primarily focus on defending against regular adversarial examples. We intend to apply these techniques to the OOD generalization problem, because OOD data may also exhibit adversarial behavior with different kinds of corruptions.

\section{On-manifold Adversarial Augmentation in the Frequency Domain}
In this section, we provide a detailed analysis of the advantages associated with on-manifold adversarial augmentation for enhancing generalization on OOD data. Furthermore, we delve into the definition of the manifold in the frequency domain.

\subsection{On-manifold Augmentation Improving OOD Generalization}
In supervised learning, an input image $\bm{x}\in\mathcal{X}$ is paired with a label $y\in\mathcal{Y}$ according to a predefined labeling rule. However, due to the inherent complexity of images, there exists a significant amount of redundant information. The prediction of the label is primarily influenced by a specific set of features $\mathcal{Z}$, while the remaining features $\mathcal{Z}_{redundant}$ are irrelevant to the prediction. The objective of many autoencoding-based methods is to learn an effective transformation that can capture and retain the primary features of an input image. Thus, it is reasonable to have an on-manifold assumption that there exists a projection from the original image space to the primary feature space, which not only reduces the dimensionality of the space but also preserves all the necessary information for classification. To be specific, the primary feature is obtained through a projection function as $\bm{z}=g(\bm{x}) \in \mathbb{R}^{d}$, where the dimension $d$ in the feature space is smaller than the dimension $d_0$ in the data space, and the loss function $\mathcal{L}(\bm{\theta},\bm{x},y)=\mathcal{L}(\bm{\theta}, g^{-1}(g(\bm{x})),y)$ ensures that the reconstructed input retains the same information as the original input for classification. The loss function is commonly assumed as a continuous, differentiable, and bounded function, \ie, $0 \leq \mathcal{L}(\bm{\theta}, \bm{x}, y) \leq M$ for constant $M$. The expected loss of a given data distribution $P$ and the label distribution $P_{y \mid \bm{x}}$ is expressed as $\mathbb{E}_{\bm{x}\sim P} [\mathbb{E}_{y\sim P_{y \mid \bm{x}}} [\mathcal{L}(\bm{\theta}, \bm{x}, y)]]$, in which $\mathbb{E}_{y\sim P_{y \mid \bm{x}}} [\mathcal{L}(\bm{\theta}, \bm{x}, y)]$ can be denoted as $\ell (\bm{\theta}, \bm{x})$. In this way, on-manifold adversarial examples are generated by modifying the latent code in a lower-dimensional feature space. A model is robust to on-manifold adversarial examples when
\begin{equation}
\label{eq:4}
\mathbb{E}_{\bm{x}\sim P_0}\left[ \sup_{\Vert \bm{\delta}_z \Vert_{\infty} \leq \epsilon_z} \left\vert \ell(\bm{\theta}, g^{-1}(g(\bm{x})+\bm{\delta}_z))-\ell(\bm{\theta},\bm{x}) \right\vert \right] \leq \tau,
\end{equation}
in which, $\bm{\delta}_z$ is the perturbation noise in the latent space, $\epsilon_z$ represents the maximum perturbation range in the latent space, $\tau$ represents the upper bound of the robustness error.

Besides, the OOD data refers to a distribution that deviates from the original distribution, and the degree of this deviation can be measured using the Wasserstein distance, as defined in~\citep{arjovsky2017wasserstein}. The OOD data resides within a distribution set $\mathcal{P}(P_0,\epsilon)=\{ P: W_p(P_0, P) \leq \epsilon\}$, in which $P_0$ is the training set distribution with compact support, $\epsilon$ represents the distance to the data distribution $P_0$ and $W_p$ represents the $W_p$-distance\footnote{The Wasserstein distance between $P$ and $Q$ can be defined as $W_{p}(P,Q)=\left(\inf_{\lambda \in \Lambda(P,Q)} \mathbb{E}_{\lambda}(\|x-y\|^{p})\right)^{1/p}$, in which $\Lambda(P,Q)$ is the coupling distribution of $P$ and $Q$, $x$ is sampled from $P$ and $y$ is sampled from $Q$.} with $p \in \{2, \infty\}$. We mainly focus on $p=\infty$. Considering $W_\infty$-distance, the generalization error to OOD data in $\mathcal{P}(P_0,\epsilon)$ can be defined as
\begin{equation}
\label{eq:5}
\gamma(\epsilon, \infty) = \left\vert \sup_{P\in \mathcal{P}(P_0,\epsilon)}\mathbb{E}_{\bm{x}'\sim P}\left[\ell (\bm{\theta}, \bm{x'})\right] - \mathbb{E}_{\bm{x}\sim P_{0}}\left[\ell (\bm{\theta}, \bm{x})\right] \right\vert,
\end{equation}
in which, $\epsilon$ represents the distribution deviation, $\bm{x'}$ is a sample from a biased distribution $P$. Since the real data distribution $P_0$ is unknown, the problem can be simplified by replacing $P_0$ with the empirical distribution $P_n$\footnote{The empirical distribution is $P_n(\cdot)=\frac{1}{n}\sum_{i=1}^{n}\bm{1}_{\{\cdot =x_i\}}$}.

Acquiring a model that demonstrates resilience to perturbations is relatively easier compared to obtaining a model that exhibits resistance to distribution bias. Thus, we can replace the problem of OOD generalization with the task of enhancing robustness against on-manifold perturbations by establishing a connection between the loss incurred by OOD data and that incurred by on-manifold adversarial examples. By combining the assumption of on-manifold behavior with the definition of OOD, we can derive the first lemma as
\begin{lemma}
\label{lemma:1}
    For $\bm{\theta}$ and $\epsilon$, we have
    \begin{equation}
    \label{eq:6}
        \sup_{P \in \mathcal{P}(P_0,\epsilon)}\mathbb{E}_{\bm{x}'\sim P}\left[\ell (\bm{\theta}, \bm{x'})\right]
        \leq
        \mathbb{E}_{\bm{x}\sim P_{0}}\left[\sup_{\|\bm{\delta}_z\|_\infty\leq \epsilon_z}\ell (\bm{\theta}, g^{-1}(g(\bm{x})+\bm{\delta}_z))\right],
    \end{equation}
    in which, $\bm{x} \in \mathbb{R}^{d_0}$, $\bm{z}=g(\bm{x}) \in \mathbb{R}^{d}$ and $\epsilon_z=\sup_{\|\bm{\delta}_x\|_\infty\leq \epsilon}\|g(\bm{x})-g(\bm{x}+\bm{\delta}_x)\|_\infty$.
\end{lemma}
\begin{proof}
    Let $P^*_\epsilon \in \arg\max_{P \in \mathcal{P}(P_0,\epsilon)}\mathbb{E}_{\bm{x}'\sim P}\left[\ell (\bm{\theta}, \bm{x'})\right]$. According to Kolmogorov’s theorem, $P^*_\epsilon$ can be a distribution of some random vector $\bm{v}$. Thus, we have
    \begin{equation}
    \label{eq:7}
        \sup_{P \in \mathcal{P}(P_0,\epsilon)}\mathbb{E}_{\bm{x}'\sim P}\left[\ell (\bm{\theta}, \bm{x'})\right]=\mathbb{E}_{\bm{v}\sim P^*_{\epsilon}}\left[\ell(\bm{\theta},\bm{v})\right].
    \end{equation}
    For $\bm{x}$ sampled from $P_0$, the random vector satisfies $\|\bm{v}-\bm{x}\|_\infty\leq \epsilon$ almost surely, which is a consequence of the definition of the $W_{\infty}$-distance. Besides, the projection function $g(\cdot)$ and its inverse $g^{-1}(\cdot)$ are both continuous. For any $\bm{v}\sim P^*_{\epsilon}$ and $\bm{x}\sim P_0$, we have
    \begin{equation}
    \label{eq:8}
        \|g(\bm{x})-g(\bm{v})\|_\infty\leq\sup_{\|\bm{v}-\bm{x}\|_\infty\leq \epsilon}\|g(\bm{x})-g(\bm{x}+(\bm{v}-\bm{x})))\|_\infty=\epsilon_z.
    \end{equation}
    This implies that if the random vector $\bm{v}$ lies within a norm ball centered at $\bm{x}$, the projection $g(\bm{v})$ also lies within a norm ball centered at $g(\bm{x})$ but with a different radius.
    
    We can construct a distribution $P_z$ in the feature space with $\bm{z}=g(\bm{x})$. Based on the basic assumption that OOD data lies on the manifold, the samples $\bm{v}$ from $P^*_{\epsilon}$ also satisfy the on-manifold property. So, we have
    \begin{align}
        \label{eq:9}
        \mathbb{E}_{\bm{v}\sim P^*_{\epsilon}}\left[\ell(\bm{\theta},\bm{v})\right]
        &=\mathbb{E}_{\bm{v}\sim P^*_{\epsilon}}\left[\ell(\bm{\theta},g^{-1}(g(\bm{v})))\right] \nonumber \\
        &\leq\sup_{P' \in \mathcal{P}(P_z,\epsilon_z)}\mathbb{E}_{\bm{z}'\sim P'}(\ell \left[\bm{\theta}, g^{-1}(\bm{z'}))\right].
    \end{align}
    Referring to the conclusion in~\citep{yi2021improved}, we have
    \begin{align}
    \label{eq:10}
        \sup_{P' \in \mathcal{P}(P_z,\epsilon_z)}\mathbb{E}_{\bm{z}'\sim P'}&\left[\ell (\bm{\theta}, g^{-1}(\bm{z'}))\right] \nonumber \\
        &=
        \mathbb{E}_{\bm{z}\sim P_{z}}\left[\sup_{\|\bm{\delta}_z\|_\infty\leq \epsilon_z}\ell (\bm{\theta}, g^{-1}(\bm{z}+\bm{\delta}_z))\right].
    \end{align}
    Together with Eq.~\ref{eq:7}, Eq.~\ref{eq:9} and Eq.~\ref{eq:10}, we get the conclusion.
\end{proof}

Lemma~\ref{lemma:1} demonstrates that the expected loss for OOD data is upper-bounded by the expected loss for on-manifold perturbations. It should be noted that these two losses are not equivalent since the perturbation in the feature space is actually enlarged. One might be concerned whether the enlarged perturbation range can make the optimization process more challenging. However, due to the limited range of perturbation, the data manifold can still be locally Euclidean. Moreover, the projection into the feature space preserves the original local structure. Hence, the impact of the enlarged perturbation range has a limited effect on the optimization process.

With the assistance of the relationship established in Lemma~\ref{lemma:1}, we can derive a quantifiable upper bound for the OOD generalization error, which is expressed in Theorem~\ref{th1} as

\begin{theorem}[OOD Upper-bound for Models Robust on On-manifold Adversarial Examples]
\label{th1}
If a model is robust to on-manifold adversarial examples given $2\epsilon_z$, $\tau$ and $p=\infty$, then for any $\epsilon \leq \epsilon_z$ with probability at least $1-e^t$,
\begin{equation}
\label{eq:11}
    \gamma(\epsilon,\infty) \leq \tau+M \sqrt{\frac{c \cdot N_z-2t}{n}},
\end{equation}
in which, $N_z=(2d)^{\frac{2D}{\epsilon_z^2}+1}$, $d$ is the dimension of the latent space, $D$ is the diameter of the compact support of feature space, $c=\log 2$ is a constant and $M$ is the upper bound of the loss function $\ell(\bm{\theta},g^{-1}(\bm{z}))$.
\end{theorem}

Detailed proof of the theorem will be provided in Appendix~A. From Theorem~\ref{th1}, we have the following observations:
\begin{itemize}
    \item Robustness to on-manifold adversarial examples is highly related to generalization on OOD data. By utilizing adversarial augmentation with on-manifold adversarial examples, we can effectively enhance OOD generalization, as the upper bound of OOD performance is constrained by on-manifold adversarial robustness.
    \item Compared with robustness to regular adversarial examples as defined in~\citep{yi2021improved}, on-manifold adversarial robustness establishes a more conservative upper bound on OOD generalization. This is primarily due to the reduced dimensionality of the latent code in the latent space compared to the original image space. Consequently, adversarial augmentation using on-manifold adversarial examples has a greater potential to enhance OOD generalization compared to regular adversarial augmentation.
\end{itemize}

\subsection{Manifold Representation in the Frequency Domain}
Despite the advantages of training models with on-manifold adversarial examples, generating such samples is a challenging task. Previous work~\citep{lin2020dual,stutz2019disentangling} has tried to approximate the data manifold with VAE models. The key idea is to train a VAE with an encoder function $g(\bm{x})$, which projects the original images into a latent code $\bm{z}$ and a decoder function $g^{-1}(\bm{z})$, which reconstructs the images from the latent code. However, approximating the data manifold using VAE models places high demands on their representational capacity, making it challenging to implement, especially for large-scale datasets. Additionally, training VAE networks requires significant computational resources and specific techniques. Therefore, it becomes necessary to replace the VAE implementation with a more stable and efficient alternative.

\begin{figure}[t]
  \centering
  \includegraphics[width=0.85\linewidth]{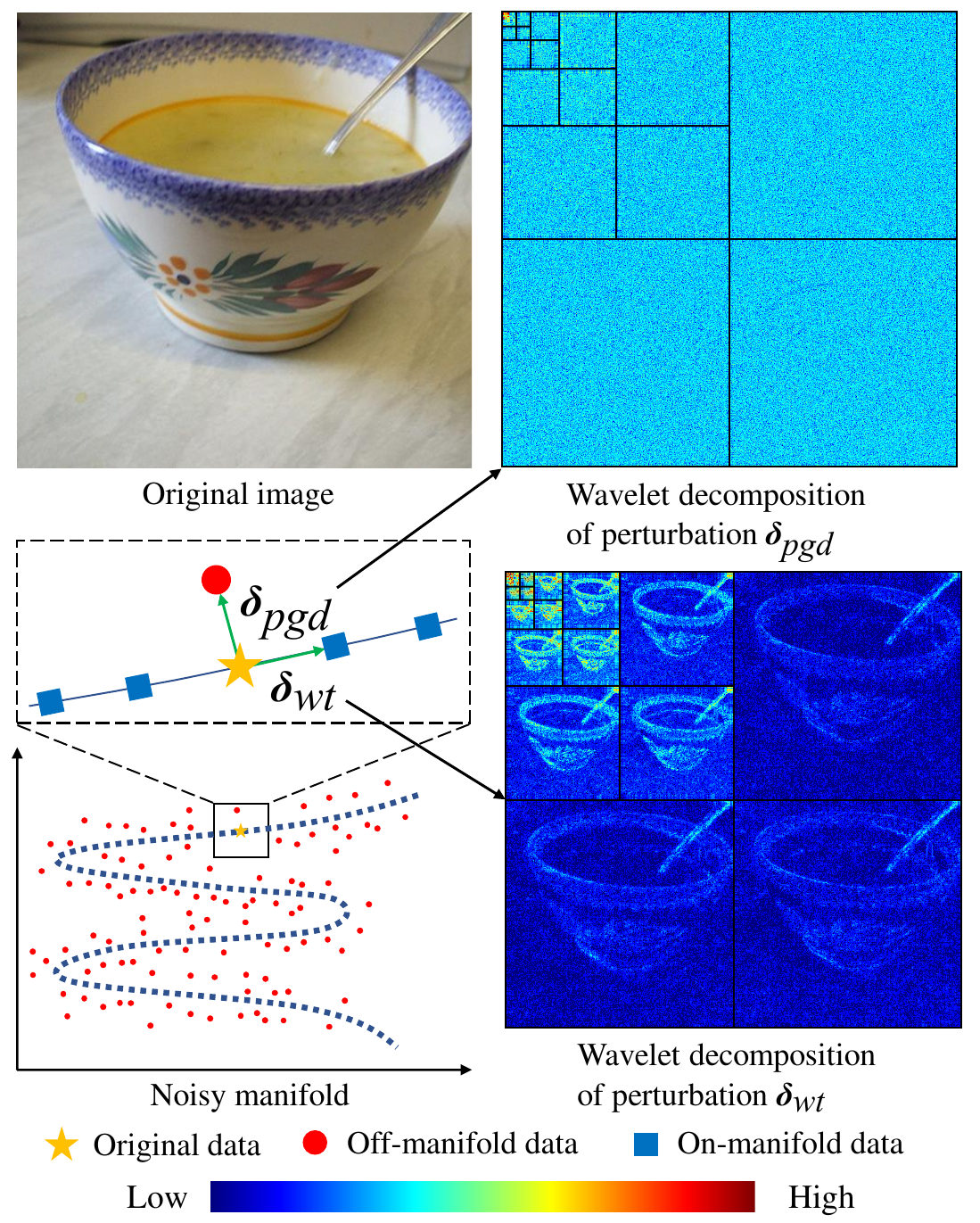}
  \caption{Illustration of on-manifold and off-manifold adversarial examples in the wavelet domain. Given the original image on the data manifold (denoted as a yellow star), a PGD attacker generates an off-manifold adversarial example (denoted as a red point) with a perturbation $\bm{\delta}_{pgd}$; our method generates an on-manifold adversarial example (denoted as a blue box) with a perturbation $\bm{\delta}_{f}$ parallel to the manifold. The wavelet decomposition of these two perturbations demonstrates that $\bm{\delta}_{f}$ is more related to the semantic meaning, while $\bm{\delta}_{pgd}$ is noise-like. 
  }
  \label{fig:2}
\end{figure}

We have noticed that projecting images into the frequency domain has a similar effect as VAE. Taking the wavelet domain as an example, the wavelet transform $\mathcal{W}(\cdot)$ and its inverse $\mathcal{W}^{-1}(\cdot)$ can be considered analogous to the encoder $g(\cdot)$ and decoder $g^{-1}(\cdot)$. Additionally, the wavelet coefficients reside in a lower-dimensional space due to the sparsity typically observed in the wavelet domain of natural images, similar to the latent code $\bm{z}$ in the latent space. Approximating the data manifold in the frequency domain offers convenience and efficiency, as the wavelet transform does not require pre-training and consumes less computational time.

We propose to construct a primary feature space containing only the non-sparse coefficients in the wavelet domain. This subspace captures the majority of the signal's energy, implying that any modifications made within this subspace will have a significant impact on semantic manipulation. With a threshold $T$ and input $\bm{x}$, the latent feature can be expressed as
\begin{equation}
\label{eq:12}
  \mathcal{Z}(\bm{x},T)=\{ \langle \bm{\psi}_i,\bm{x} \rangle \mid \langle \bm{\psi}_i,\bm{x} \rangle \geq T,\bm{\psi}_i \in \Psi \},
\end{equation}
in which, $\Psi$ contains all the wavelet bases given the mother wavelet; and $T$ is a predefined threshold to filter out the sparse coefficients. In this way, the manifold projection function is defined as a composition of the wavelet transform $\mathcal{W}$ and a filtering operation that selects the non-sparse coefficients using an indicator function $\mathcal{F}(\bm{z})=\bm{z} \cdot \bm{1}_{\{ z_i \geq T \} }$, \ie, $g(\bm{x})=\mathcal{F} \circ \mathcal{W}(\bm{x})$. The inverse of $g(\bm{x})$ corresponds to the inverse wavelet transform, \ie, $g^{-1}(\bm{z})=\mathcal{W}^{-1}(\bm{z})$. As shown in Fig.~\ref{fig:2}, regular adversarial attacks generate off-manifold adversarial examples with noise-like patterns in the frequency domain. The perturbation $\bm{\delta}_{pgd}$ introduced by a PGD attacker captures the gradient of a classifier but lacks semantic meaning. We can obtain an on-manifold perturbation $\bm{\delta}_{f}$ by projecting $\bm{\delta}_{pgd}$ onto the pre-defined manifold.

Then, we provide the definition of the models robust to on-manifold adversarial examples in the frequency domain as
\begin{equation}
\label{eq:13}
\mathbb{E}_{\bm{x}\sim P_0}\left[ \sup_{\Vert \bm{\delta}_f \Vert_{p} \leq \epsilon_f} \left\vert \ell(\bm{\theta}, \mathcal{W}^{-1}(\bm{z}_f+\bm{\delta}_f))-\ell(\bm{\theta},\bm{x}) \right\vert \right] \leq \tau,
\end{equation}
in which, latent vector $\bm{z}_f$ is obtained by $\bm{z}_f=\mathcal{F} \circ \mathcal{W}(\bm{x})$, $\epsilon_f$ represents the maximum perturbation range in the latent space, $\tau$ represents the upper bound of the expected loss. It is shown that the approximation of the data manifold in the frequency domain can be regarded as a special case of general manifold estimation, which means that models that are robust to on-manifold adversarial examples in the frequency domain also have an upper bound on OOD generalization.

\begin{figure*}[t]
  \centering
   \includegraphics[width=0.98\linewidth]{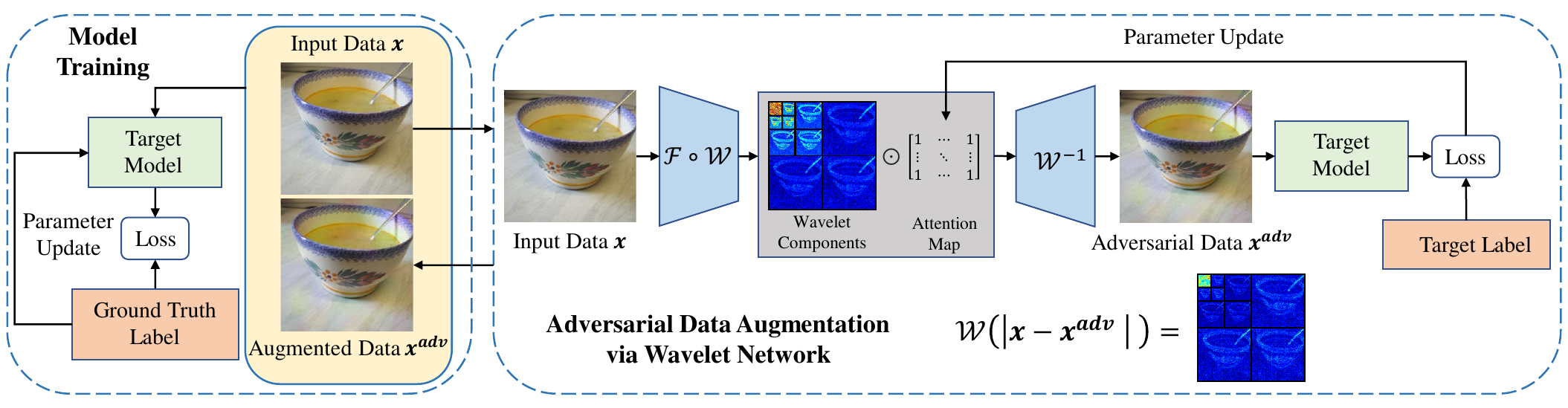}
  \vspace{-1ex}
   \caption{The overall pipeline of adversarial data augmentation with AdvWavAug, yielding an improved model generalization. To obtain on-manifold adversarial examples $\bm{x}^{adv}$, we send the input image $\bm{x}$ into an adversarial augmentation module, in which $\bm{x}$ is projected into the wavelet domain by wavelet transformation $\mathcal{W}$ and its inverse $\mathcal{W}^{-1}$. We design an attention map to receive the perturbations backpropagated from the loss. The wavelet decomposition of the perturbations ensures that AdvWavAug generates on-manifold adversarial examples. During model training, the original input $\bm{x}$ is augmented with the adversarial examples $\bm{x}^{adv}$, and the augmented data is used to train the target model.
   }
   \label{fig:3}
\end{figure*}

\section{Boosting Model Generalization via AdvWavAug}
\label{sec:4}
In this section, we introduce the training scheme combined with our on-manifold augmentation module AdvWavAug.
\subsection{Problem Formulation}
\label{sec:4.1}
The definition of on-manifold adversarial examples in the frequency domain requires determining the threshold value $T$ in Eq.~\ref{eq:12}. An alternative and more adaptive approach to identifying non-sparse coefficients is by replacing traditional additive perturbations with multiplicative perturbations. To generate on-manifold adversarial examples, we use the multiplicative perturbation in the frequency domain as
\begin{equation}
\label{eq:14}
    \mathop{\arg \max} \limits_{\|\Tilde{\bm{\delta}}_f\|_p \leq \Tilde{\epsilon}_f} \ell(\bm{\theta}, \mathcal{W}^{-1}(\bm{z}_f \odot (\bm{1}+\Tilde{\bm{\delta}}_f))),
\end{equation}
where $\bm{z}_f=\mathcal{F} \circ \mathcal{W}(\bm{x})$, $\mathcal{W}(\cdot)$ is the wavelet transform and $\mathcal{W}^{-1}(\cdot)$ is its inverse, $\bm{1}+\Tilde{\bm{\delta}}_f$ represents an adversarially modified attention map (wavelet components that contribute more to model degradation are amplified). It is noted that we introduce the adversarial perturbations in the frequency domain rather than the image domain, which facilitates preserving the on-manifold characteristics of the examples.

The perturbation range can be derived from a regular perturbation constraint by zeroing out all wavelet coefficients below the threshold $T$ as zeros, as follows
\begin{equation}
\label{eq:15}
    \| \Tilde{\bm{\delta}}_f \|_{p} \leq \frac{\| \mathcal{W}(\bm{\delta}) \|_{p}}{\sqrt[p]{n}T} \leq \frac{PQ\epsilon_f}{\sqrt[p]{n}T}=\Tilde{\epsilon}_f,
\end{equation}
in which $n$ is the total number of non-sparse coefficients, $P$ and $Q$ are positive constants. It is noted that $\Tilde{\epsilon}_f$ serves as a new upper bound for $\| \Tilde{\bm{\delta}}_f \|_{p}$. In practice, the perturbation $\| \Tilde{\bm{\delta}}_f \|_{p}$ that satisfies the constraint $\Tilde{\epsilon}_f$ is a sufficient condition. The assignment to each wavelet scale is determined accordingly. A more detailed analysis is provided in Appendix~B.

AdvProp~\citep{xie2020adversarial} has been shown to be an effective training scheme for improving model generalization with adversarial examples. However, the AdvProp framework has limited effectiveness with off-manifold adversarial examples because the perturbations are not semantically meaningful. To address this limitation, we enhance the AdvProp framework by introducing on-manifold adversarial augmentation, yielding the resultant problem formulation as
\begin{equation}
\label{eq:16}
\mathop{\arg \min} \limits_{\bm{\theta}} \mathbb{E}_{\bm{x} \sim P_0}\left[\ell(\bm{\theta},\bm{x})
+\mathop{\arg \max} \limits_{\Vert \Tilde{\bm{\delta}}_f \Vert_p \leq \Tilde{\epsilon}_f}\ell(\bm{\theta}, \mathcal{W}^{-1}(\bm{z}_f \odot (\bm{1}+\Tilde{\bm{\delta}}_f)))\right],
\end{equation}
where $\bm{z}_f=\mathcal{F} \circ \mathcal{W}(\bm{x})$ and clean data $\bm{x}$ is augmented with our on-manifold adversarial examples $\mathcal{W}^{-1}(\bm{z}_f \odot (\bm{1}+\Tilde{\bm{\delta}}_f))$. In this way, we can find the real pitfalls of a well-trained model on clean data distribution and improve the model generalization.

\subsection{Implementation Details}
\label{sec:4.2}
In this section, we will introduce the training process, which integrates our augmentation module AdvWavAug with the AdvProp structure. The whole network is shown in Fig.~\ref{fig:3}. The input image is first decomposed into different scales by the fast wavelet transformation proposed by~\citep{mallat1999wavelet}. Then, the wavelet coefficients are multiplied with an attention map, which acts as a frequency filter in signal processing. Next, the product of wavelet coefficients and the attention map is sent to the corresponding inverse wavelet transformation. The reconstructed adversarial examples are fed into the target model to calculate the gradient to the attention map. Finally, with the optimal attention map, we can generate on-manifold adversarial examples, which serve as intentional data augmentation to boost the model's generalization to OOD data. AdvWavAug is designed to project the gradient map onto the data manifold, making it a versatile module that can be integrated into other training frameworks. The detailed algorithm is shown in Algorithm~\ref{alg:A}.

\begin{algorithm}[H]
\caption{Pseudo code of adversarial training with AdvWavAug for $T_e$ epochs, $\alpha$ step size, $S$ adversarial steps and a dataset of size $N$}\label{alg:A}
\begin{algorithmic}
\STATE \textbf{Data:} A set of clean images $\mathcal{X}$ and labels $\mathcal{Y}$
\STATE \textbf{Result:} Network parameter $\bm{\theta}$
\FOR{$t=1\cdots T_e$}
\FOR{$n=1\cdots N$}
\STATE Sample $\bm{x}^c \in \mathcal{X}$ with label $y \in \mathcal{Y}$
\STATE $\bm{z}_f=\mathcal{W}(\bm{x}^c)$
\STATE $\Tilde{\bm{\delta}}_f=\bm{0}$
\FOR{$s=1\cdots S$}
\STATE $\Tilde{\bm{\delta}}_f=\Tilde{\bm{\delta}}_f+\alpha \cdot \nabla_{\Tilde{\bm{\delta}}_f}\ell(\bm{\theta}, \mathcal{W}^{-1}(\bm{z}_f \odot (\bm{1}+\Tilde{\bm{\delta}}_f)))$
\ENDFOR
\STATE $\bm{x}^a=\mathcal{W}^{-1}(\bm{z}_f \odot (\bm{1}+\Tilde{\bm{\delta}}_f))$;
\STATE $\bm{\theta}=\bm{\theta}-\nabla_{\bm{\theta}}(\ell(\bm{\theta},\bm{x}^c)+\ell(\bm{\theta},\bm{x}^a))$
\ENDFOR
\ENDFOR
\end{algorithmic}
\label{alg1}
\end{algorithm}

\section{Experiments}
\label{sec:5}

\subsection{Experiments Setup}

\subsubsection{Datasets}
To verify the effectiveness of our adversarial augmentation module, we conduct all training experiments using the standard ImageNet~2012~\citep{russakovsky2015imagenet} training dataset. In addition to evaluating the performance for clean accuracy on ImageNet~2012 validation dataset, we also evaluate the generalization on its distorted version for OOD data.
\begin{itemize}
\setlength{\itemsep}{0pt}
    \item ImageNet-A~\citep{hendrycks2021natural}: The ImageNet-A dataset consists of 7,500 real-world, unmodified, and naturally occurring examples, drawn from some challenging scenarios.
    \item ImageNet-C~\citep{hendrycks2019benchmarking}: The ImageNet-C dataset consists of 19 different corruption types grouped into noise, blur, weather, digital and extra corruption. And each type has five levels of severity
    \item ImageNet-R~\citep{hendrycks2021many}: The ImageNet-R dataset contains 16 different renditions (\eg art, cartoons, deviantart, \etc).
\end{itemize}

\begin{table}[h]
\begin{center}
\begin{minipage}{230pt}
\caption{Wavelet settings of the AdvWavAug module with different step sizes in different frequency bands.}
\label{tab:1}%
\begin{tabular}{@{}lccccccc@{}}
\toprule
& $H_{1}$ & $H_{2}$ & $H_{3}$ & $H_{4}$ & $H_{5}$ & $H_{6}$ & $L$ \\
\midrule
S1 & 0.50&0.07&0.05&0.03&0.02&0.010&0.001  \\
S2 & 0.40&0.06&0.04&0.03&0.02&0.010&0.001  \\
S3 & 0.30&0.05&0.04&0.03&0.02&0.015&0.015  \\
S4 & 0.10&0.30&0.05&0.03&0.02&0.010&0.010  \\
S5 & 0.09&0.09&0.13&0.15&0.17&0.150&0.150  \\
S6 & 0.09&0.09&0.09&0.11&0.13&0.150&0.170  \\
\bottomrule
\end{tabular}
\end{minipage}
\end{center}
\end{table}

\subsubsection{Architectures}
To verify the attack performance of AdvWavAug, we conduct model training on ResNet~\citep{he2016deep}, Swin Transformer~\citep{liu2021swin} and Vision Transformer~\citep{dosovitskiy2020image}. For ResNet, ResNet~50 (Res50), ResNet~101 (Res101), and ResNet~152 (Res152) are chosen to verify the performance across different model sizes. For Swin Transformer, we select the Swin Transformer Tiny (SwinT), Swin Transformer Small (SwinS), and Swin Transformer Base (SwinB). For Vision Transformer, we select the Vision Transformer Base (ViTB), Vision Transformer Large (ViTL), and Vision Transformer Huge (ViTH). In order to accommodate the transformer architectures that lack batch normalization layers, we replace the original batch normalization in AdvProp with layer normalization.

\subsubsection{Augmentation Module} 
We choose AdvWavAug as the default augmentation module to generate adversarial examples during training and compare its performance with three other types of augmentation modules: PGD adversarial augmentation (vanilla AdvProp), Gaussian noise augmentation and VQ-VAE augmentation~\citep{razavi2019generating}.

During model training with AdvProp, we used a one-step attack iteration for generating adversarial examples. The PGD attacker is set up with perturbation range $\epsilon=1/255$, number of iterations $n=1$, and attack step size $\alpha=1/255$.

During model training with Gaussian noise augmentation, we add Gaussian noise with mean 0.0 and std 0.001 to the clean images and take these augmented images as auxiliary inputs.

During model training with VQ-VAE augmentation, we first approximate the data manifold by training a VQ-VAE model based on ImageNet. Then, we generate on-manifold adversarial examples following the process in~\citep{stutz2019disentangling}. The perturbation noise is added in the latent space with a one-step attack and a step size of 0.007. Finally, the on-manifold images are sent into the auxiliary channel.

During model training with our AdvWavAug, we adopt wavelet decomposition with 6 layers and the sym8 mother wave as in~\citep{daubechies1993ten}. Considering that modifications in high-frequency bands may have a lower contribution to the overall visual quality, we balanced the perturbation ranges in different frequency bands by using larger step sizes in higher-frequency bands. Except for the ablation study of different wavelet settings, the basic setting of the step sizes during model training is Setting 3 (S3) in Tab.~\ref{tab:1}, in which $H_{1}$ to $H_{6}$ denote multi-scale decomposition from high to low-frequency bands, and $L$ denotes the lowest-frequency band. To illustrate the effect of balancing step sizes in different frequency bands, we gradually shift the attention area from higher to lower frequency bands, \ie, from Setting 1 (S1) to Setting 6 (S6) in Tab.~\ref{tab:1}.

\begin{table*}[h]\small
\begin{center}
\begin{minipage}{0.85\textwidth}
\caption{Comparison of generalization with different training methods on different datasets, including ImageNet, ImageNet-A ImageNet-R, and ImageNet-C. We compare the baseline model, Gaussian Augmentation, AdvProp, and AdvWavAug.}
\label{tab:2}
\begin{tabular*}{\textwidth}{@{\extracolsep{\fill}}llcccc@{\extracolsep{\fill}}}
\toprule%
Model & Method & ImageNet & ImageNet-A & ImageNet-R & ImageNet-C \\
\cmidrule{3-6}
&& Top-1 Acc. \textcolor{red}{$\uparrow$} & Top-1 Acc. \textcolor{red}{$\uparrow$} & Top-1 Acc. \textcolor{red}{$\uparrow$} & mCE \textcolor{red}{$\downarrow$} \\
\midrule
\multirow{4}*{Res50}
& Baseline &   76.3&   2.5&   35.9&   77.4\\
& Gaussian &   77.0&   2.5&   36.8&   75.7\\
& AdvProp  &   77.4&   3.5&   37.8&   69.9\\
& AdvWavAug&\bf77.5&\bf4.1&\bf39.3&\bf69.0\\
\midrule
\multirow{4}*{Res101}
& Baseline &   78.2&   5.5&   39.4&   70.8\\
& Gaussian &   78.7&   6.1&   40.4&   69.4\\
& AdvProp  &\bf79.2&   8.1&\bf42.2&   65.5\\
& AdvWavAug&\bf79.2&\bf8.6&\bf42.2&\bf63.7\\
\midrule
\multirow{4}*{Res152}
& Baseline &   78.8&    7.3&   40.6&   69.1\\
& Gaussian &   79.0&    8.6&   40.8&   67.3\\
& AdvProp  &   79.8&   11.4&   43.6&\bf62.3\\
& AdvWavAug&\bf80.5&\bf13.0&\bf45.9&   63.2\\
\midrule
\multirow{4}*{SwinT}
& Baseline &   78.6&   13.9&   38.1&   67.0\\
& Gaussian &   79.4&\bf16.6&   39.3&   63.2\\
& AdvProp  &   79.5&   15.3&   40.6&   61.9\\
& AdvWavAug&\bf79.7&   15.8&\bf44.9&\bf56.1\\
\midrule
\multirow{4}*{SwinS}
& Baseline &   80.6&   21.2&   41.0&   61.1\\
& Gaussian &   80.2&   20.3&   40.9&   58.2\\
& AdvProp  &   80.7&   20.3&   42.1&   56.9\\
& AdvWavAug&\bf81.6&\bf25.8&\bf47.7&\bf48.9\\
\bottomrule
\end{tabular*}
\end{minipage}
\end{center}
\end{table*}

\subsection{Adversarial Propagation with \textbf{AdvWavAug}}
\label{sec:5.2}
\subsubsection{Training Hyperparameters}
\label{sec:5.2.1}
For the AdvProp training of ResNets, we follow these settings: SGD optimizer with momentum 0.9 and weight decay 5e-5; initial learning rate 0.2 with cosine learning rate scheduler; epoch 105 including 5 warm-up epochs; batch size 256; random horizontal flipping and random cropping as the basic data augmentations. For AdvProp training of Swin Transformers, we follow the same basic training settings as in~\citep{liu2021swin}, except that 
the training epoch is set to 105 including 5 warm-up epochs, and batch size is set to 512 (for Swin-Tiny) and 256 (for Swin-Small) due to limited computing resources. Incidentally, we simply combine our proposed AdvWavAug and Augmix strategy with no Jensen-Shannon divergence consistency loss in Sec.~\ref{sec:5.2.4}.

For the normal adversarial training part in Sec.~\ref{sec:5.3.2}, to maintain training efficiency, we set up PGD attackers with $\epsilon=2/255$, number of iteration $n=1$, and attack step size $\alpha=2/255$. Then we set up the one-step AdvWavAug attackers with competitive attack success rates.

For the VQ-VAE model training, we adopt the settings in~\citep{razavi2019generating}: mean square error loss; latent loss weight 0.25; learning rate 3e-4 with cycle scheduler; epoch 560; batch size 256.

For model fine-tuning with MAE, we adopt the settings in~\citep{he2021masked}.

\begin{figure}[!t]
  \centering
  \includegraphics[width=0.98\linewidth]{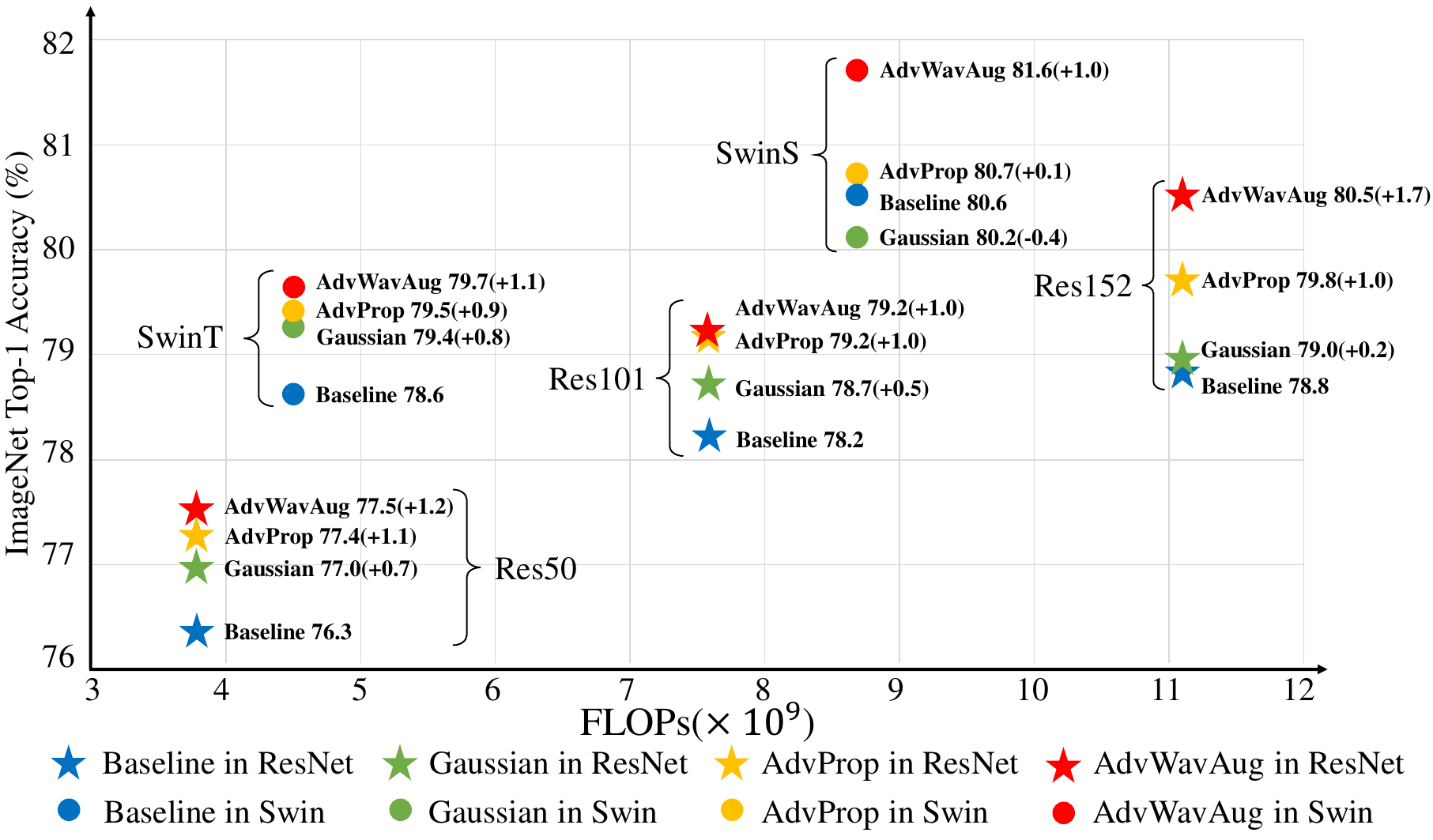} 
  \caption{Generalization comparison on ImageNet. AdvWavAug boosts model performance over the original AdvProp with PGD attacker on ImageNet. Improvement to larger models is more significant. Our best result is based on Swin-Small (Swin-S), \ie, 81.6\% Top-1 Acc. on ImageNet.} 
  \label{fig:advprop_clean} 
\end{figure}

\subsubsection{Comparison with Gaussian Augmentation and Vanilla AdvProp}
In this part, we compare our method with two typical augmentations: Gaussian noise augmentation (transform-based) and AdvProp (adversary-based). To ensure a fair comparison, we modify the training epochs of Swin Transformers to match the epochs of ResNets. We train all the models according to the training hyperparameters in Sec.~\ref{sec:5.2.1}.

\noindent \textbf{ImageNet Validation Results} In Fig.~\ref{fig:advprop_clean}, we have shown the experimental results on the ImageNet validation set. We verify our method on ResNet and Swin Transformer. As observed, our proposed AdvWavAug method outperforms both Gaussian augmentation and the original AdvProp with PGD on all tested architectures and model sizes. These results validate our fundamental belief that on-manifold adversarial examples can improve model generalization better than off-manifold adversarial examples. This is because on-manifold adversarial examples improve generalization by directly finding the real pitfalls of a model to the data manifold, instead of introducing noise-like adversarial patterns.

In addition, we observe that the performance improvement is correlated with the model's capacity. Our method has a greater improvement on models with larger capacities. For example, a ResNet 152 trained with our method achieves 80.5\% Top-1 Acc., resulting in a 0.7\% performance gain over AdvProp and a 1.7\% gain over the vanilla training baseline.

\noindent \textbf{Generalization on OOD Datasets}
We also evaluate the models on more challenging distorted ImageNet datasets: ImageNet-A, ImageNet-C, and ImageNet-R. The results are summarized in Tab.~\ref{tab:2}. Based on the results of Gaussian augmentations, it is evident that adversary-based augmentations generally outperform a single type of corrupted augmentation. This is because adversary-based augmentations explore the boundary of a classifier, benefiting generalization on unseen corruptions, while the improvement of Gaussian noise augmentation is minor. Considering the limited improvement from Gaussian augmentations, we focus on comparing our method with the vanilla AdvProp. Taking Swin-Small as an example, AdvWavAug achieves the best performance with performance gains of 5.5\% on ImageNet-A, 8.0\% on ImageNet-C, and 5.6\% on ImageNet-R, compared with the original AdvProp. These results demonstrate that our AdvWavAug improves model generalization on these distorted datasets, indicating that on-manifold adversarial examples can further enhance generalization on OOD data. The results support our claim that models robust to on-manifold adversarial examples have an upper bound on OOD samples, which is commonly smaller compared with models robust to off-manifold examples.

\begin{table*}[h]\small
\begin{center}
\begin{minipage}{0.85\textwidth}
\caption{Comparison of generalization with AdvWavAug and VQ-VAE augmentation on different datasets, including ImageNet, ImageNet-A, ImageNet-R and ImageNet-C. The training settings of Res50, Res101, Res152, SwinT and SwinS are based on the aligned ones.}
\label{tab:3}
\begin{tabular*}{\textwidth}{@{\extracolsep{\fill}}llcccc@{\extracolsep{\fill}}}
\toprule%
Model & Method & ImageNet & ImageNet-A & ImageNet-R & ImageNet-C \\
\cmidrule{3-6}
&& Top-1 Acc. \textcolor{red}{$\uparrow$} & Top-1 Acc. \textcolor{red}{$\uparrow$} & Top-1 Acc. \textcolor{red}{$\uparrow$} & mCE \textcolor{red}{$\downarrow$} \\
\midrule
\multirow{2}*{Res50}
& VQ-VAE   &   77.4&\bf4.5&\bf41.0&\bf68.0 \\
& AdvWavAug&\bf77.5&   4.1&   39.3&   69.0 \\
\midrule
\multirow{2}*{Res101}
& VQ-VAE   &\bf79.2&\bf10.2&\bf44.1&\bf62.4 \\
& AdvWavAug&\bf79.2&    8.6&   42.2&   63.7 \\
\midrule
\multirow{2}*{Res152}
& VQ-VAE   &   80.3&\bf14.5&\bf47.1&\bf58.2 \\
& AdvWavAug&\bf80.5&   13.0&   45.9&   63.2 \\
\midrule
\multirow{2}*{SwinT}
& VQ-VAE   &   79.4&   15.4&\bf45.7&   57.0 \\
& AdvWavAug&\bf79.7&\bf15.8&   44.9&\bf56.1 \\
\midrule
\multirow{2}*{SwinS}
& VQ-VAE   &\bf81.7&\bf27.5&\bf50.7&   49.1 \\
& AdvWavAug&   81.6&   25.8&   47.7&\bf48.9 \\
\bottomrule
\end{tabular*}
\end{minipage}
\end{center}
\end{table*}

\begin{table*}[h]\small
\begin{center}
\begin{minipage}{0.85\textwidth}
\caption{Comparison of generalization with AdvWavAug and VQ-VAE augmentation on different datasets, including ImageNet, ImageNet-A, ImageNet-R and ImageNet-C. The training settings of SwinT, SwinS and SwinB are based on the standard ones.}
\label{tab:4}
\begin{tabular*}{\textwidth}{@{\extracolsep{\fill}}llccccc@{\extracolsep{\fill}}}
\toprule%
Model & Method & ImageNet & ImageNet-A & ImageNet-R & ImageNet-C \\
\cmidrule{3-6}
&& Top-1 Acc. \textcolor{red}{$\uparrow$} & Top-1 Acc. \textcolor{red}{$\uparrow$} & Top-1 Acc. \textcolor{red}{$\uparrow$} & mCE \textcolor{red}{$\downarrow$} \\
\midrule
\multirow{3}*{SwinT}
& Baseline &   81.2&   20.7&   41.9&   62.0 \\
& VQ-VAE   &   81.2&   21.1&   42.3&   60.3 \\
& AdvWavAug&\bf81.4&\bf22.2&\bf47.0&\bf53.2 \\
\midrule
\multirow{3}*{SwinS}
& Baseline &   83.2&   32.1&   45.3&   54.9 \\
& VQ-VAE   &   82.9&   31.6&\bf52.8&   46.2 \\
& AdvWavAug&\bf83.4&\bf35.0&   49.4&\bf45.8 \\
\midrule
\multirow{3}*{SwinB}
& Baseline &   83.5&   34.4&   47.2&   54.5 \\
& VQ-VAE   &   83.3&   33.4&   49.4&   50.2 \\
& AdvWavAug&\bf83.9&\bf38.3&\bf51.1&\bf44.9 \\
\bottomrule
\end{tabular*}
\end{minipage}
\end{center}
\end{table*}

\begin{table*}[h]\small
\begin{center}
\begin{minipage}{0.85\textwidth}
\caption{Combination with another augmentation technique AugMix on different datasets, including ImageNet, ImageNet-A, and ImageNet-R. We combine our $\rm AdvWavAug$ with AugMix, and compare its performance with the original AugMix on ResNet-50.}
\label{tab:5}
\begin{tabular*}{\textwidth}{@{\extracolsep{\fill}}llccccc@{\extracolsep{\fill}}}
\toprule%
Model & Method & ImageNet & ImageNet-A & ImageNet-R & ImageNet-C \\
\cmidrule{3-6}
&& Top-1 Acc. \textcolor{red}{$\uparrow$} & Top-1 Acc. \textcolor{red}{$\uparrow$} & Top-1 Acc. \textcolor{red}{$\uparrow$} & mCE \textcolor{red}{$\downarrow$} \\
\midrule
\multirow{4}*{Res50}
& Baseline &76.3&2.5&35.9&77.4 \\
\cmidrule{2-6}
& AugMix &77.5&3.8&41.1&65.3 \\
\cmidrule{2-6}
& AdvWavAug &   77.5&   4.3&   39.5&   68.0 \\
& +AugMix   &\bf77.7&\bf5.6&\bf41.3&\bf63.3 \\
\bottomrule
\end{tabular*}
\end{minipage}
\end{center}
\end{table*}

\begin{table*}[h]\small
\begin{center}
\begin{minipage}{0.85\textwidth}
\caption{Comparison of generalization with MAE and MAE+AdvWavAug on different datasets, including ImageNet, ImageNet-A and ImageNet-R.}
\label{tab:6}
\begin{tabular*}{\textwidth}{@{\extracolsep{\fill}}llccccc@{\extracolsep{\fill}}}
\toprule%
Model & Method & ImageNet & ImageNet-A & ImageNet-R & ImageNet-C \\
\cmidrule{3-6}
&& Top-1 Acc. \textcolor{red}{$\uparrow$} & Top-1 Acc. \textcolor{red}{$\uparrow$} & Top-1 Acc. \textcolor{red}{$\uparrow$} & mCE \textcolor{red}{$\downarrow$} \\
\midrule
\multirow{2}*{ViTB}
& Baseline &   83.6&   35.9&   48.3&   51.7 \\
& AdvWavAug&\bf83.9&\bf37.9&\bf50.8&\bf47.0 \\
\midrule
\multirow{2}*{ViTL}
& Baseline &   85.9&   57.1&   59.9&   41.8 \\
& AdvWavAug&\bf86.2&\bf58.9&\bf61.3&\bf38.1 \\
\midrule
\multirow{2}*{ViTH}
& Baseline &   86.9& 68.2&   64.4&   33.8 \\
& AdvWavAug&\bf87.1& \bf 68.3&\bf65.4&\bf32.6 \\
\bottomrule
\end{tabular*}
\end{minipage}
\end{center}
\end{table*}

\begin{figure}[!t]
  \centering
  \includegraphics[width=0.85\linewidth]{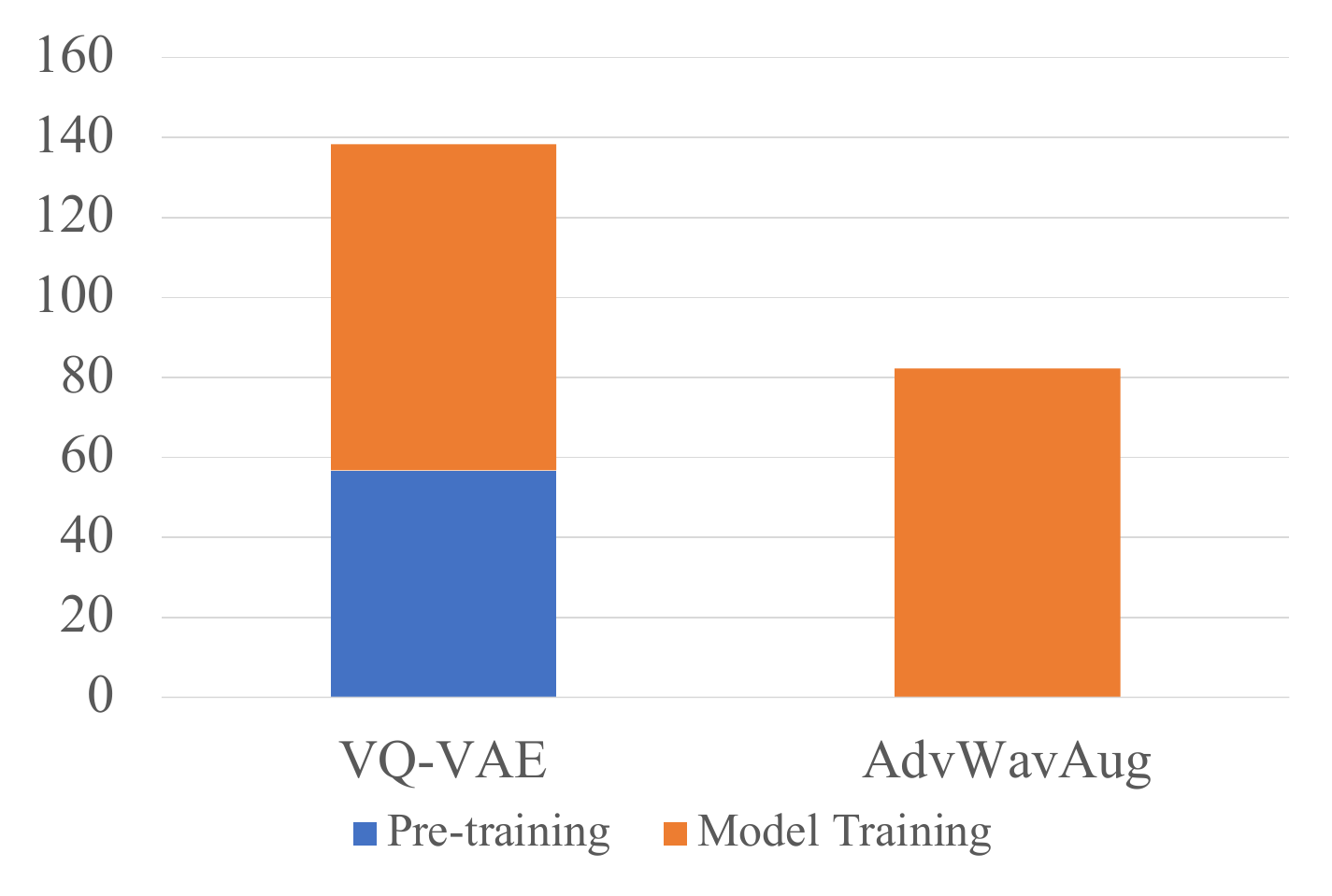} 
  \caption{Comparison of computation cost with AdvWavAug and VQ-VAE augmentation based on Swin Transformer Small.} 
  \label{fig:time} 
\end{figure}

\subsubsection{Comparison with VQ-VAE Augmentation}
We also train a VQ-VAE model to approximate the data manifold of ImageNet. With the assistance of the VQ-VAE model, we can generate on-manifold adversarial examples as shown in~\citep{stutz2019disentangling}. To compare the performance of our method with VQ-VAE augmentation, we conducted experiments with aligned settings, where the training epochs of Swin Transformers were set to 105. The results are summarized in Tab.~\ref{tab:3}.

Then, to achieve SOTA results on Swin Transformers, we follow the standard settings as defined in~\citep{liu2021swin}, and Swin Transformer Base model is added. The results are summarized in Tab.~\ref{tab:4}, in which the baseline models are taken directly from the official implementation in~\citep{liu2021swin}. Taking Swin Transformer Base with standard settings as an example, our method improves OOD generalization on ImageNet by 0.6\%, ImageNet-A by 4.9\%, ImageNet-R by 1.7\%, and ImageNet-C by 5.3\%. It has been shown that the implementation of the data manifold with AdvWavAug achieves competitive OOD generalization compared to the VQ-VAE implementation based on aligned training settings, while our method outperforms the VQ-VAE implementation based on standard Swin Transformers. From the aligned settings, it seems that the convergence time of the VQ-VAE augmentation is faster than that of our AdvWavAug. However, the final converged results of the VQ-VAE implementation are inferior to those of the AdvWavAug, highlighting the challenge of accurately approximating the data manifold with VQ-VAE. Generally, approximating the data manifold, especially on large-scale datasets, places great demands on the representational capacity of VAEs. In contrast, our method overcomes it in a more stable and simple way by leveraging frequency analysis. We have also compared the total training time of both augmentations in Fig.~\ref{fig:time}. We conduct the experiments based on Swin Transformer Small with 32 NVIDIA GeForce RTX 3090 GPUs, each with 24 GB memory. The computation cost of AdvWavAug is nearly $59.4\%$ lower than that of the VQ-VAE augmentation because our method requires no pre-training process. Although the number of convolution operations of our method is smaller than the VQ-VAE implementation, our method shows no obvious advantage during the model training process. That is because the wavelet transforms have not been fully optimized in the PyTorch implementation, while the convolutions in VQ-VAE are optimized based on the PyTorch framework. We believe the training time will be reduced once the operations in AdvWavAug are fully optimized. Overall, these findings demonstrate that AdvWavAug can approximate the data manifold more efficiently by leveraging the sparsity in the frequency domain. Considering both the augmentation effect and efficiency, AdvWavAug serves as a promising alternative to VQ-VAE augmentation.

\subsubsection{Integration with Other Data Augmentations}
\label{sec:5.2.4}
We have also combined our module with other data augmentation techniques, such as AugMix~\citep{hendrycks2019augmix}. The effectiveness of the combined version on transformers has already been verified by the results in Tab.~\ref{tab:4} because AugMix has been integrated into the training process of Swin Transformers in the baseline setting. Therefore, we only provide the results based on ResNet 50 here. The results are shown in Tab.~\ref{tab:5}, in which the AugMix results are directly taken from~\citep{hendrycks2019augmix}. Our method exhibits a similar augmentation effect to common data augmentation techniques, such as AugMix. Our AugMix+$\rm AdvWavAug$ has better performance gains with 0.2\% on ImageNet, 1.8\% on ImageNet-A 0.2\% on ImageNet-R, and 2.0\% on ImageNet-C, compared to using AugMix alone. It indicates that the generalization improvement provided by our method does not conflict with common data augmentations. The reason behind this is that our method enhances generalization by identifying the worst-case scenarios (adversarial examples) within the manifold, while common data augmentations improve generalization by increasing the diversity of the training data.

\begin{table*}[h]\small
\begin{center}
\begin{minipage}{0.85\textwidth}
\caption{Comparison of generalization with AdvWavAug and AdvWavAug+PGD on different datasets, including ImageNet, ImageNet-A, ImageNet-R and ImageNet-C.}
\label{tab:7}
\begin{tabular*}{\textwidth}{@{\extracolsep{\fill}}llccccc@{\extracolsep{\fill}}}
\toprule%
Model & Method & ImageNet & ImageNet-A & ImageNet-R & ImageNet-C \\
\cmidrule{3-6}
&& Top-1 Acc. \textcolor{red}{$\uparrow$} & Top-1 Acc. \textcolor{red}{$\uparrow$} & Top-1 Acc. \textcolor{red}{$\uparrow$} & mCE \textcolor{red}{$\downarrow$} \\
\midrule
\multirow{2}*{Res50}
& AdvWavAug &\bf77.5&   4.1&   39.3&   69.0 \\
& +PGD      &   77.4&\bf4.3&\bf39.5&\bf68.0 \\
\midrule
\multirow{2}*{Res101}
& AdvWavAug &   79.2&   8.6&   42.2&   63.7 \\
& +PGD      &\bf79.5&\bf8.9&\bf43.4&\bf63.0 \\
\midrule
\multirow{2}*{Res152}
& AdvWavAug &\bf80.5&\bf13.0&\bf45.9&   63.2 \\
& +PGD      &   79.1&    9.6&   42.1&\bf62.7 \\
\midrule
\multirow{2}*{SwinT}
& AdvWavAug &\bf79.7&   15.8&\bf44.9&\bf56.1 \\
& +PGD      &\bf79.7&\bf16.8&   43.9&   56.6 \\
\midrule
\multirow{2}*{SwinS}
& AdvWavAug &\bf81.6&\bf25.8&\bf47.7&\bf48.9 \\
& +PGD      &   80.5&   20.9&   44.9&   53.4 \\
\bottomrule
\end{tabular*}
\end{minipage}
\end{center}
\end{table*}

\begin{table*}[h]\small
\begin{center}
\begin{minipage}{0.85\textwidth}
\caption{Comparison of normal adversarial training scheme on ResNet50 with adversarial examples generated by PGD and our AdvWavAug. We have chosen four datasets, including Top-1 Acc. (\%) on ImageNet, Top-1 Acc. (\%) on ImageNet-A, Top-1 Acc. (\%) on ImageNet-R and mCE (\%) on ImageNet-C.}
\label{tab:8}
\begin{tabular*}{\textwidth}{@{\extracolsep{\fill}}llcccc@{\extracolsep{\fill}}}
\toprule%
Model & Method & ImageNet & ImageNet-A & ImageNet-R & ImageNet-C \\\cmidrule{3-6}%
&& Top-1 Acc. \textcolor{red}{$\uparrow$} & Top-1 Acc. \textcolor{red}{$\uparrow$} & Top-1 Acc. \textcolor{red}{$\uparrow$} & mCE \textcolor{red}{$\downarrow$} \\\midrule
\multirow{3}*{Res50}
& Baseline &\bf76.3&2.5&35.9&77.4 \\
\cmidrule{2-6}
& PGD-AT &75.9&2.6&\bf74.7&38.6 \\
\cmidrule{2-6}
& AdvWavAug-AT &76.1&\bf2.9&74.0&\bf37.0 \\
\bottomrule
\end{tabular*}
\end{minipage}
\end{center}
\end{table*}

\subsubsection{Integration into Masked Autoencoders}
It has been shown that masked autoencoders can learn the structural aspects of objects in an image. On-manifold adversarial examples are more closely related to the semantic meanings of input images, providing a deeper understanding of image structure. Therefore, we integrate our method into the fine-tuning process of MAE method. The baseline models are taken directly from~\citep{he2021masked}. The results are shown in Tab.~\ref{tab:6}, in which our method has improved ViTL on ImageNet by 0.3\%, ImageNet-A by 1.8\%, and ImageNet-R by 1.4\% ImageNet-C by 3.7\%. It has been shown that our method can further improve the OOD generalization of Vision Transformers pre-trained with MAE. This improvement is attributed to AdvWavAug's ability to explore the boundary of a classifier within the data manifold.

\begin{figure}[!t]
  \centering
  \includegraphics[width=0.98\columnwidth]{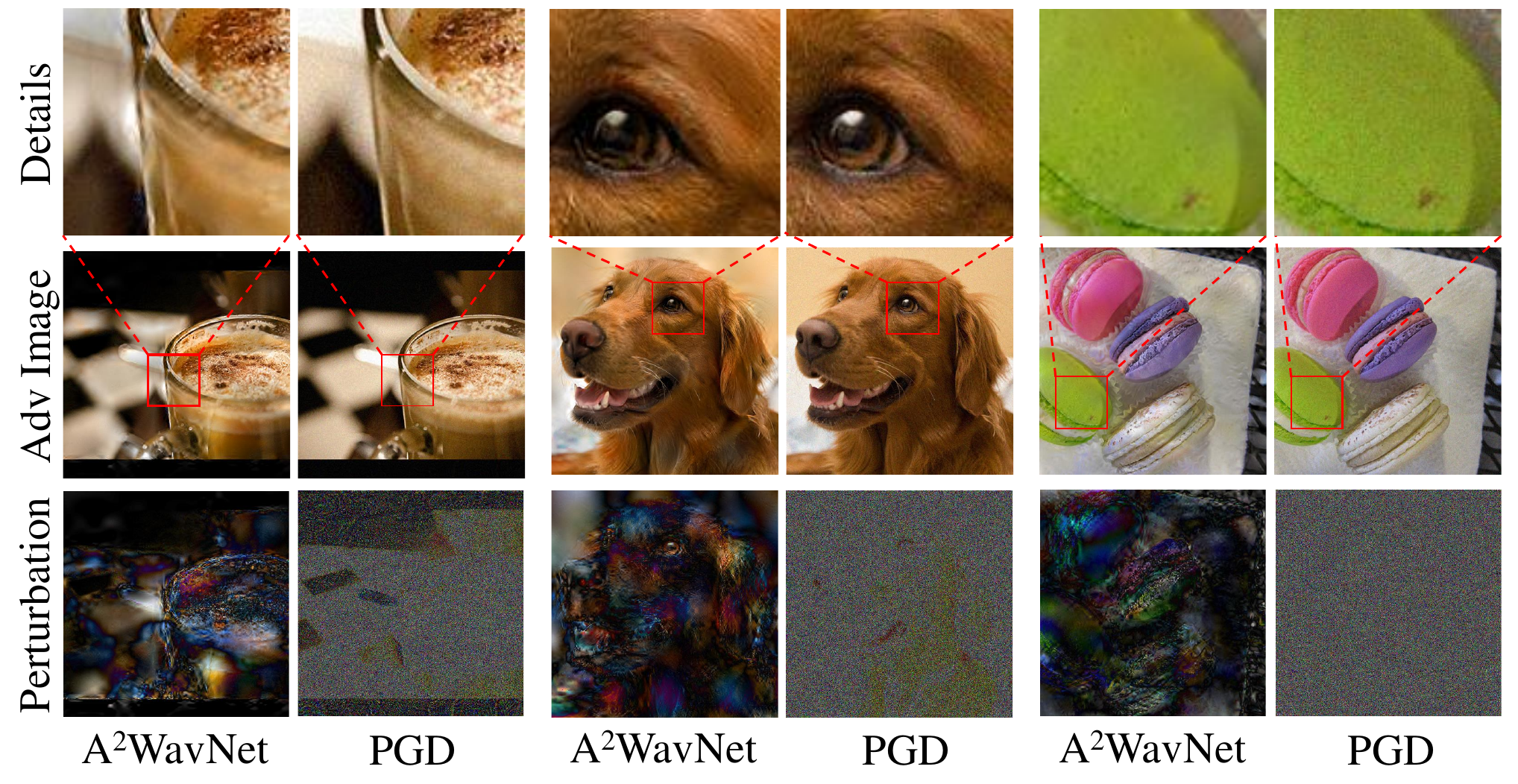} 
  \caption{Visualization of adversarial examples generated by AdvWavAug and PGD. It can be seen that adversarial examples of our method are more natural, while PGD adversarial examples have many noise-like patterns or off-manifold perturbations.} 
  \label{fig:visual} 
\end{figure}

\subsubsection{Visualization of Adversarially Augmented Data}
In Fig.~\ref{fig:visual}, we present examples of adversarial images generated using AdvWavAug, which exhibit more natural details and perturbations that are more semantically meaningful. In Appendix~C, we compare AdvWavAug with PGD based on the image quality of the adversarial examples, as measured by metrics such as FID~\citep{heusel2017gans} and LPIPS~\citep{zhang2018unreasonable}. The results demonstrate that AdvWavAug can generate adversarial examples with desirable attack performance and a closer distance to the data manifold than PGD.

\subsection{Ablation Study}
\label{sec:5.3}
\subsubsection{Integration with Off-manifold Adversarial Augmentation}
Although on-manifold adversarial augmentation generally outperforms off-manifold augmentation, it remains unknown how the combination of both augmentations would perform. To investigate this, we randomly select adversarial examples from AdvWavAug and PGD attacker as augmentation samples, and the detailed results are shown in Tab.~\ref{tab:7}.

Interestingly, the performance of AdvWavAug +PGD decreases in some of the tested models. It means introducing off-manifold adversarial examples cannot stably improve model generalization and may even conflict with on-manifold adversarial examples. This phenomenon aligns with our analysis that off-manifold adversarial examples cannot find the real pitfalls of a model. Furthermore, altering the model structure can undermine the generalization improvement achieved by off-manifold adversarial examples. Our method can find the real pitfalls of a model, and consistently improve generalization across different model architectures.

\subsubsection{Normal Adversarial Training}
\label{sec:5.3.2}
In this experiment, we perform adversarial training using the same objective function as described in Eq.~\eqref{eq:3} and Eq.~\eqref{eq:16}, where both clean and adversarial data are fed to the network within a mini-batch. Tab.~\ref{tab:8} shows the experimental results. Expect for ImageNet-R, our AdvWavAug outperforms PGD on other datasets under normal adversarial training. These results demonstrate that our method can effectively enhance the overall generalization of the model in the context of normal adversarial training. This superiority can be attributed to the fact that on-manifold adversarial examples remain within the data manifold and are thus more suitable for normal adversarial training.

\begin{table}[h]
\begin{center}
\begin{minipage}{205pt}
\caption{Comparison of AdvWavAug generalization with different wavelet settings. We have chosen ResNet-50 as the base model and compared the Top-1 Acc. (\%). From Setting 1 (S1) to Setting 6 (S6), the weights of perturbations will gradually shift from high-frequency bands to low-frequency bands.}
\label{tab:9}
\begin{tabular}{@{}lcccccc@{}}
\toprule%
Model & \multicolumn{6}{@{}c@{}}{Wavelet Setting} \\
\cmidrule{2-7}
&1&2&3&4&5&6 \\
\midrule
Res50 &76.7&76.8&\bf77.5&77.1&77.0&76.9 \\
\bottomrule
\end{tabular}
\end{minipage}
\end{center}
\end{table}

\subsubsection{Training with Different Wavelet Settings}
We also investigated the impact of modified frequency bands on the generalization performance. We gradually shift the attention from high-frequency bands to low-frequency bands by varying the wavelet setting from Setting 1 (S1) to Setting 6 (S6) in Tab.~\ref{tab:1}. Results in Tab.~\ref{tab:9} show that wavelet Setting 3 (S3), which primarily modifies higher frequency bands while making subtle modifications to the remaining bands, achieves the best performance on ResNet-50. This outcome can be attributed to the difficulty in controlling modifications on low-frequency bands, which may lead to adversarial examples that deviate from the data manifold. Therefore, the chosen Setting 3 (S3) is a relatively balanced setting.

\section{Conclusion}
\label{sec:6}
In this paper, we have proposed AdvWavAug, a new data augmentation algorithm based on on-manifold adversarial examples in the frequency domain to improve OOD generalization, followed by an AdvProp training scheme to minimize the loss function of both the clean samples and the on-manifold adversarial examples.

We have provided theoretical proof that models resilient to on-manifold adversarial examples have an upper bound on OOD generalization, which is commonly smaller compared to models resilient to off-manifold adversarial examples. This indicates that on-manifold adversarial augmentation is more effective in enhancing OOD generalization. We have conducted extensive experiments to verify the effectiveness of our method, resulting in SOTA performance on two transformer-based architectures including Swin Transformers and Vision Transformers pre-trained by MAE.

Moving forward, our future work will focus on two key aspects. First, we will build an adversarial dataset based on the adversarial examples generated by our method, which can be used to test the robustness of a classifier to more natural adversarial examples. Besides, our method can be regarded as a new paradigm to generate adversarial examples, which opens up opportunities for its integration into various research areas that involve adversarial examples. For example, in certain studies that combine adversarial augmentations with contrastive learning, such as those mentioned in~\citep{ho2020contrastive}, our method could potentially replace the original adversarial module, thereby offering a new approach to enhance performance.

\section*{Acknowledgments}
This work was supported by the NSFC Projects (Nos. 62076147, U19A2081, U19B2034, U1811461), Alibaba Group through Alibaba Innovative Research Program, a grant from Tsinghua Institute for Guo Qiang, and the High Performance Computing Center, Tsinghua University.

\noindent \begin{center} {\large  \textbf{Appendix for: Improving Model Generalization by On-manifold Adversarial Augmentation in the Frequency Domain}} \end{center}
\begin{appendix}
\setcounter{corollary}{0}

In this appendix, we provide additional technical details of the proposed method. In~\ref{appendix:a}, we provide a detailed proof for the upper-bound of OOD generalization based on on-manifold adversarial robustness in Sec. 3.1. Then, in~\ref{appendix:b} we offer a detailed analysis of the perturbation range in 
the frequency domain, which is defined in Sec. 4.1. In~\ref{appendix:c}, we provide a detailed analysis of the on-manifold adversarial examples generated by our method, considering both image quality and attack performance. In~\ref{appendix:d}, we provide detailed results on different corruptions in the ImageNet-C dataset.
\section{Detailed Proof for Relationship Between OOD Generalization and On-manifold Adversarial Robustness}
\label{appendix:a}
Lemma 1 demonstrates that models robust to on-manifold perturbations establish an upper bound for models robust to regular perturbations. Hence, we can draw similar conclusions as in Theorem 1 referring to~\citep{yi2021improved}. The detailed proof begins by considering the definition of a cover for a set of points.

\begin{definition}[\citep{wainwright2019high}]
    To a point set $\mathcal{Z}$, the $\epsilon$-cover in norm $\|\cdot\|_\infty$ is defined as $\mathcal{Z}_c=\{z_1, \cdots , z_N\}$, in which, for any $z \in Z$, there exists $z_i \in \mathcal{Z}_c$ such that $\|z-z_i\|_\infty \leq \epsilon$. The covering number can be defined as $N(\mathcal{Z},\epsilon, \|\cdot\|_\infty)=\inf\{N \in \mathbb{N}~\mid$ $\mathcal{Z}_c=\{z_1, \cdots , z_N\}$ is a $\epsilon$-cover  of $\mathcal{Z}$ in $\|\cdot\|_\infty$ norm\}.
\end{definition}

Then, we come to the proof of Theorem 1, which is borrowed from~\citep{yi2021improved}.

\begin{proof}
As we transfer the image data space into the latent space, we only need to cover the latent space, which has a reduced dimension than the image data space. Just the same as the conclusion in~\citep{yi2021improved}, we can obtain an inequality that $N(\mathcal{Z}, \epsilon_z, \|\cdot\|_\infty) \leq (2d)^{(2D/\epsilon_z^2+1)}=N_z$, in which $\mathcal{Z}$ is the feature in the latent space, $d$ is the number of dimensions of the latent code, $D$ is the diameter of the support in latent space. Then, we can construct an $\epsilon_z$ cover in norm $\|\cdot\|_\infty$ of $\mathcal{Z}$ with pairwise disjoint sets $\mathcal{C}=\{C_1,\cdots,C_{N_z}\}$. For each $c_i,c_j \in \mathcal{C}$, we have $\|c_i-c_j\|_{\infty} \leq \epsilon_z$. We can construct the probability distribution in the latent space $\bm{z}\sim P_z$ with $\bm{z}=g(\bm{x})$. Due to Lemma 1, we have
\begin{align}
    \gamma(\epsilon,p)\leq&\Biggl\vert\mathbb{E}_{\bm{z}\sim P_{z}}\left[\sup_{\|\bm{\delta}_z\|_\infty\leq \epsilon_z}\ell (\bm{\theta}, g^{-1}(\bm{z}+\bm{\delta}_z))\right]
    -\mathbb{E}_{\bm{x}\sim P_{0}}\left[\ell (\bm{\theta}, \bm{x})\right] \Biggr\vert \\ \nonumber
    =&\Biggl\vert \sum_{i=1}^{N_z} \mathbb{E}_{\bm{z}\sim P_z} \left[ \sup_{\|\bm{\delta}_z\|_\infty\leq \epsilon_z}\ell (\bm{\theta}, g^{-1}(\bm{z}+\bm{\delta}_z)) \mid \bm{z} \in C_i \right] 
    \cdot P_0(C_i)\\ \nonumber
    &~~~~~~- \frac{1}{N}\sum_{j=1}^{N}\ell(\bm{\theta},g^{-1}(g(\bm{x}_i)))\Biggr\vert \\ \nonumber
    \leq& \Biggl\vert \sum_{i=1}^{N_z} \mathbb{E}_{\bm{z}\sim P_z} \left[ \sup_{\|\bm{\delta}_z\|_\infty\leq \epsilon_z}\ell (\bm{\theta}, g^{-1}(\bm{z}+\bm{\delta}_z)) \mid \bm{z} \in C_i \right]\cdot \frac{N_{C_i}}{N}  \\ \nonumber
    &~~~~~~- \frac{1}{N}\sum_{j=1}^{N}\ell(\bm{\theta},g^{-1}(g(\bm{x}_i)))\Biggr\vert +M\sum_{k=1}^{N_z}\left\vert P_0(C_k)-\frac{N_{C_k}}{N} \right\vert \\ \nonumber
    =&\Biggl\vert \frac{1}{N} \sum_{i=1}^{N_z} \sum_{\bm{z}_j\in C_i}\Biggl[\mathbb{E}_{\bm{z}\sim P_z} \left[ \sup_{\|\bm{\delta}_z\|_\infty\leq \epsilon_z}\ell (\bm{\theta}, g^{-1}(\bm{z}+\bm{\delta}_z)) \mid \bm{z} \in C_i \right] \\ \nonumber
    &~~~~~~-\ell(\bm{\theta},g^{-1}(\bm{z}_j))\Biggr]\Biggr\vert+M\sum_{k=1}^{N_z}\left\vert P_0(C_k)-\frac{N_{C_k}}{N} \right\vert \\ \nonumber
    \leq& \Biggl\vert\frac{1}{N} \sum_{i=1}^{N_z} \sum_{\bm{z}_j \in C_i} \sup_{\bm{z} \in B(C_j,\epsilon_z)} \vert \ell (\bm{\theta}, g^{-1}(\bm{z}))-\ell(\bm{\theta}, g^{-1}(\bm{z}_i)) \vert\Biggr\vert \\ \nonumber
    &~~~~~~+M\sum_{k=1}^{N_z}\left\vert P_0(C_k)-\frac{N_{C_k}}{N} \right\vert \\ \nonumber
    \leq& \frac{1}{N} \sum_{i=1}^N \sup_{\|\bm{\delta}_z\|_\infty \leq 2\epsilon_z} \vert \ell (\bm{\theta}, g^{-1}(\bm{z}_i+\bm{\delta}_z))-\ell(\bm{\theta}, g^{-1}(\bm{z}_j)) \vert \\ \nonumber
    &~~~~~~+M\sum_{k=1}^{N_z}\left\vert P_0(C_k)-\frac{N_{C_k}}{N} \right\vert \\ \nonumber
    \leq& \tau+M\sum_{k=1}^{N_z}\left\vert P_0(C_k)-\frac{N_{C_k}}{N} \right\vert,
\end{align}
in which, $B(C_j,\epsilon_z)=\{\bm{p}_m \mid \exists c_n \in C_j, \|p_m-c_n\|_{\infty}\leq\epsilon_z\}$, $N_{C_k}$ represents the number of data points in $C_k$. For the same proposition in~\citep{wellner2013weak}, we get the conclusion in Theorem 1.
\end{proof}

\section{Detailed Analysis for Adversarial Augmentation of AdvWavAug}
\label{appendix:b}
In Sec. 4.1, we formulate the multiplicative perturbation in the frequency domain as Eq.~14. The range of perturbation can be derived from the constraints on regular perturbations. By comparing the latent code of a regular adversarial example $\bm{z}^{adv}_r$ with our proposed adversarial example $\bm{z}^{adv}_f$ in the frequency domain

\begin{align}
    \bm{z}^{adv}_r&=\mathcal{W}(\bm{x})+\mathcal{W}(\bm{\delta}), \label{eq:22}\\
    \bm{z}^{adv}_f&=\mathcal{W}(\bm{x}) \odot (\textbf{1}+\Tilde{\bm{\delta}_{f}}), \label{eq:23}
\end{align}
the magnitude of perturbation in Eq.~\eqref{eq:22} should be lower than that in Eq.~\eqref{eq:23}. By setting all wavelet coefficients below the threshold $T$ to zero, we can get
\begin{equation}
\label{eq:24}
    \| \Tilde{\bm{\delta}_{f}} \|_{p} \leq \frac{\| \mathcal{W}(\bm{\delta}) \|_{p}}{\sqrt[p]{n} T},
\end{equation}
in which, $n$ is the total number of non-sparse coefficients.

Since we perform an orthogonal wavelet transform, according to Parseval's theorem, the $L_2$-norm of the coefficients in the frequency domain and the spatial domain are equal. Therefore, there exist two positive constants $P$ and $Q$ such that
\begin{align}
\label{eq:25}
    \| \mathcal{W}(\bm{\delta}) \|_{p} &\leq P \| \mathcal{W}(\bm{\delta}) \|_{2} \\ \nonumber
    &= P \| \bm{\delta}_f \|_{2} \\ \nonumber
    &\leq PQ\| \bm{\delta}_f \|_{p} \\ \nonumber
    &\leq PQ\epsilon_f.
\end{align}

By combining Eq.~\eqref{eq:24} and Eq.~\eqref{eq:25}, we can obtain the constraint of $\Tilde{\bm{\delta}_{f}}$ as
\begin{equation}
\label{eq:26}
    \| \Tilde{\bm{\delta}_{f}} \|_{p} \leq \frac{PQ\epsilon_f}{\sqrt[p]{n}T} = \Tilde{\epsilon}_f,
\end{equation}
in which, $\Tilde{\epsilon}_f$ is the new upper bound of $\| \Tilde{\bm{\delta}_{f}} \|_{p}$.

\section{Image Quality of On-manifold Adversarial Examples}
\label{appendix:c}
In this section, we compare the effectiveness of the on-manifold adversarial examples generated by our AdvWavAug with regular adversarial examples, \eg PGD~(\cite{madry2017towards}) and DIM~(\cite{xie2019improving}), in terms of both image quality and attack performance.
\subsection{White-box Attack Settings}
During the white-box attack, we mainly compare the performance of the PGD and AdvWavAug attack module. We set up the PGD attack with $\epsilon=1/255$, iteration number $n=1$ and step size $\alpha=1/255$; AdvWavAug module with iteration number $n=1$, step size as Setting 3 (S3) in Tab.~1. To maintain the continuity of attacks and prevent the clip operation from affecting it, we employ a technique that sets the constraints in different scales by adjusting the step size at each iteration, rather than the final constraint. This approach allows us to effectively limit the attack perturbation $\Tilde{\epsilon}$ within an adaptive constraint.

\begin{figure*}[tp]
  \centering
  \includegraphics[width=0.96\linewidth]{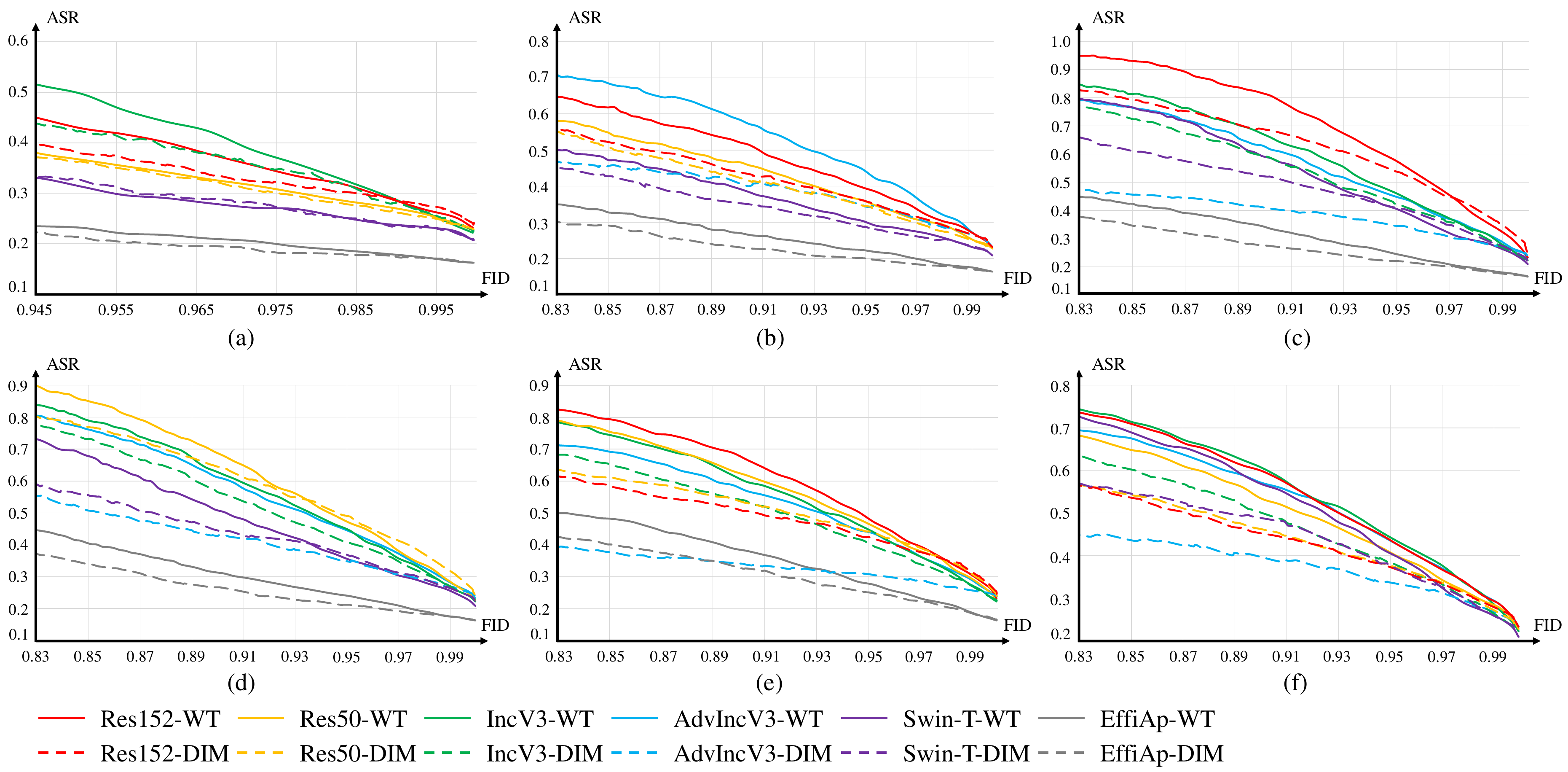} 
  \caption{The ASR-FID curves of various target models. The curves in (a)-(f) exhibit the transferability to other models with respectively AdvIncV3, IncV3, Res50, Res152, Swin-T and EffiAp as the target models. Higher positions on the curves indicate better attack performance while maintaining the same image quality. It is evident that the adversarial examples generated by AdvWavAug (solid lines) outperform those generated by DIM (dashed lines) when targeting the same model.} 
  \label{fig:app1} 
\end{figure*}

\begin{figure*}[tp]
  \centering
  \includegraphics[width=0.96\linewidth]{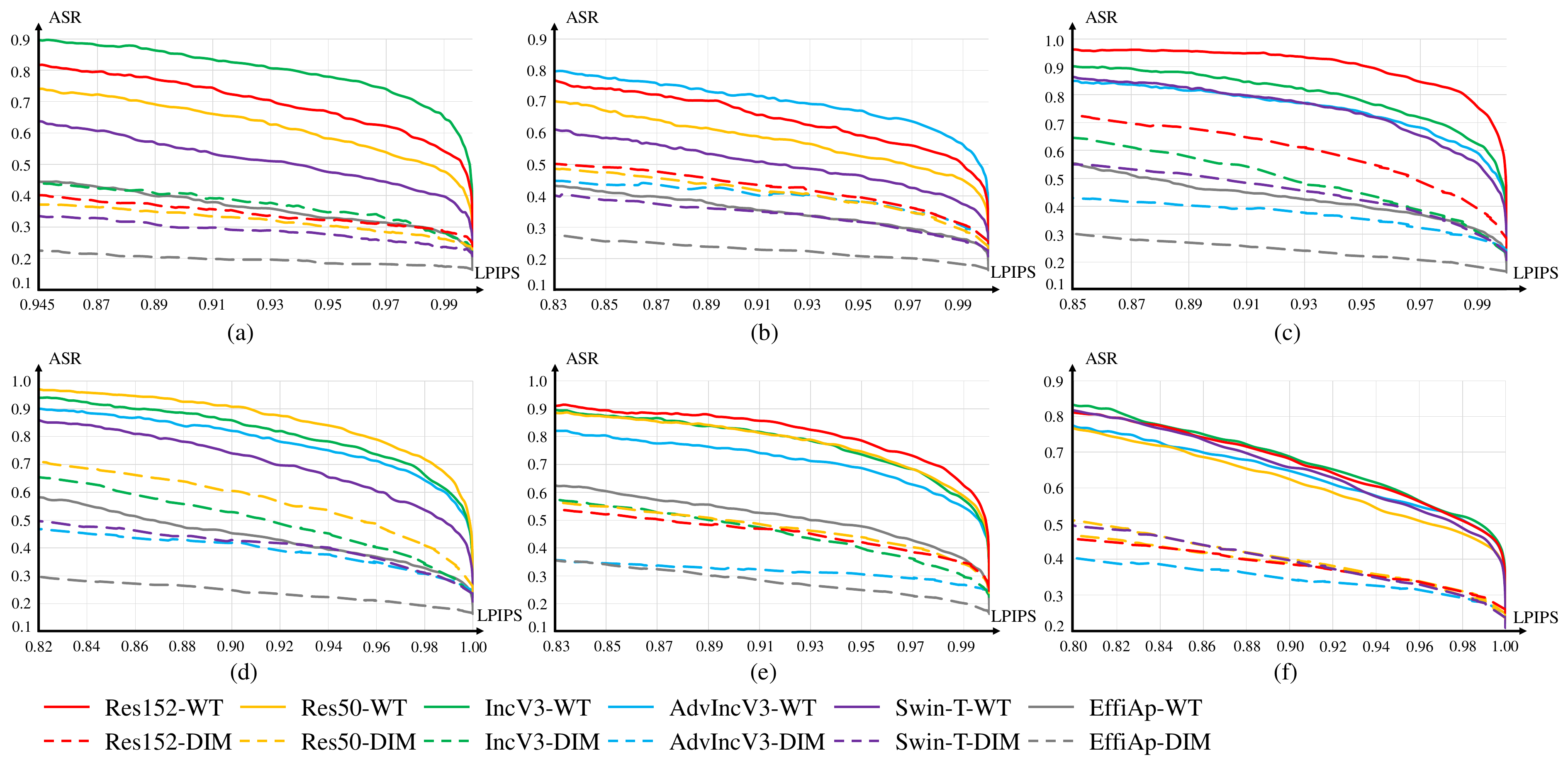} 
  \caption{The ASR-LPIPS curves of various target models. The curves in (a)-(f) exhibit the transferability to other models with respectively AdvIncV3, IncV3, Res50, Res152, Swin-T and EffiAp as the target models. Higher positions on the curves indicate better attack performance while maintaining the same image quality. It is evident that the adversarial examples generated by AdvWavAug (solid lines) outperform those generated by DIM (dashed lines) when targeting the same model.} 
  \label{fig:app2} 
\end{figure*}

\subsection{Black-box Attack Settings}
To demonstrate that our method can be plugged into any attack method and make the adversarial examples closer to the manifold, we also conduct black-box attacks. In these attacks, we combine our AdvWavAug method with DIM as the target attack method, forming AdvWavAug-DIM. Since perturbations in the frequency domain cannot be directly compared to perturbations in the spatial domain, we compare the attack performance with similar image quality, which is further described in~\ref{sec:c.5}. We set up DIM with a 0.7 probability to transform the input images, including random scaling, cropping and padding. The maximum perturbation range for DIM is set as $\epsilon=16/255$. We aim to investigate the relationship between the attack performance, represented by the attack success rate (ASR), and image quality metrics such as Fréchet Inception Distance (FID)~(\cite{heusel2017gans}) and Learned Perceptual Image Patch Similarity (LPIPS)~(\cite{zhang2018unreasonable}). To observe the attack process comprehensively, where the ASR gradually increases while the image quality gradually decreases, we plot a curve that depicts the relationship. Each point on the curve represents an intermediate adversarial example, with the x-axis denoting the image quality score and the y-axis denoting the ASR. To generate adversarial examples with different perturbation ranges, we gradually adjust our perturbation settings. However, abruptly changing the settings can result in a discontinuous attack process. To overcome this, we use a small step size and a large iteration number, allowing us to simulate different perturbation ranges indirectly. As the iterations progress, the perturbation range gradually increases, while maintaining a smooth attack process. Therefore, we set the number of iterations to $n=100$ and the step size as $16/(255*100)$. The maximum perturbation range $\epsilon=16/255$ is reached at the end of the iterations. Additionally, we remove the clipping operations to ensure that the perturbations remain within the dynamic perturbation ranges defined by our settings. Since black-box performance is independent of the model training process, we select the parameters that yield the best performance. The AdvWavAug-DIM module is configured with the same settings as DIM, step size Setting 3 (S3) from Tab.~1, and adaptive constraints.

\begin{table}[h]\small
\begin{center}
\caption{The white-box attack performance, including ASR (\%), FID (\%), LPIPS (\%) and SCORE (\%), of PGD vs AdvWavAug for different models.}
\label{tab:white_box}%
\begin{tabular}{@{}llcccc@{}}
\toprule
Model & Attack & ASR & FID & LPIPS & SCORE \\
\midrule
\multirow{2}*{Res50}
& PGD &62.7&\bf96.1&97.8&58.9 \\
& AdvWavAug &\bf70.7&95.9&\bf99.3&\bf67.3 \\
\midrule
\multirow{2}*{Res152}
& PGD &82.7&95.9&97.6&77.4 \\
& AdvWavAug &\bf83.5&\bf96.1&\bf99.3&\bf79.6 \\
\midrule
\multirow{2}*{IncV3}
& PGD &71.8&72.8&97.2&50.8 \\
& AdvWavAug &\bf72.9&\bf91.9&\bf99.3&\bf66.5 \\
\midrule
\multirow{2}*{AdvIncV3}
& PGD &40.0&97.4&98.9&38.5 \\
& AdvWavAug &\bf40.8&\bf97.6&\bf99.6&\bf39.6 \\
\midrule
\multirow{2}*{EffiAP}
& PGD &37.5&79.9&96.2&28.8 \\
& AdvWavAug &\bf37.7&\bf96.6&\bf98.4&\bf35.8 \\
\midrule
\multirow{2}*{SwinT}
& PGD &55.2&\bf98.5&99.0&53.8 \\
& AdvWavAug &\bf60.5&96.9&\bf99.3&\bf58.2 \\
\bottomrule
\end{tabular}
\end{center}
\end{table}

\subsection{Evaluation Metrics for Attacks}
\label{sec:c.3}
  We adopt various image quality metrics, including FID and LPIPS, to measure the distance to the manifold. In this section, we provide the calculation formulation for the attack success rate (ASR) and the normalized image quality metrics, including FID and LPIPS. The normalization process follows the guidelines of the CVPR 2021 challenge: Unrestricted Adversarial Attacks on ImageNet. Although we modify some parameters to enhance the differences, the scores remain within the range of $[0, 1]$. Suppose that the original $N$ clean images $X=\{x_1,x_2,\cdots,x_n\}$ are perturbed as adversarial examples $\hat{X}=\{\hat{x}_1,\hat{x}_2,\cdots,\hat{x}_n\}$. The ASR measures the attack ability as
  \begin{equation}
      {\rm ASR}=\frac{\| \{ \hat{x} \mid \mathcal{F}(\hat{x}) \ne y\} \|}{N},
  \end{equation}
in which, $\mathcal{F}(\hat{x})$ is the output of a classifier.

One of the image quality metrics is FID, which represents the naturalness of a generated image. The normalized version is expressed as
\begin{equation}
    {\rm FID}=\sqrt{1-\frac{\min(fid(X,\hat{X}),ubf)}{ubf}},
\end{equation}
in which, $fid$ is calculated based on the original setting in~\citep{heusel2017gans}, $ubf$ is the upper bound of the original $fid$ score. We set $ubf=10$ for white-box attacks and $ubf=200$ for black-box attacks.

Another image quality metric is LPIPS, which represents the perceptual distance between two images. The normalized LPIPS score is expressed as
\begin{equation}
    {\rm LPIPS}=\sqrt{1-2*(\min(\max(lpips,ubl),lbl))},
\end{equation}
in which, $lpips$ is the original score referred to~\citep{zhang2018unreasonable}, $ubl$ is the upper bound and $lbl$ is the lower-bound of the original $lpips$ score. We set $lbl=0.0$, $ubl=0.5$ in both white-box attacks and $lbl=0.1$, $ubl=0.6$ in black-box attacks.

The total score, denoted as SCORE, takes into account both the attack performance and the image quality scores, which is expressed as
\begin{equation}
    \rm SCORE=100 *ASR*FID*LPIPS.
\end{equation}

\subsection{Experimental Results on White-box Attacks}
The results in Tab.~\ref{tab:white_box} show that our AdvWavAug method outperforms PGD in generating adversarial examples, as indicated by higher values of ASR, FID, LPIPS, and SCORE. These results highlight the effectiveness of AdvWavAug in generating adversarial examples with improved attack performance and closer proximity to the data manifold, compared to PGD.

\renewcommand\arraystretch{1.4}
\begin{sidewaystable*}\small
\begin{center}
\begin{minipage}{\textheight}
\caption{Comparison of generalization with different training methods on different corruptions in ImageNet-C. We compare the baseline model, Gaussian Augmentation, AdvProp and AdvWavAug.}
\label{tab:2.2}
\begin{tabular*}{\textheight}{@{\extracolsep{\fill}}lllccccccccccccccc@{\extracolsep{\fill}}}
\toprule%
& Method & mCE \textcolor{red}{$\downarrow$} & \multicolumn{3}{@{}c@{}}{Noise} & \multicolumn{4}{@{}c@{}}{Blur} & \multicolumn{4}{@{}c@{}}{Weather} & \multicolumn{4}{@{}c@{}}{Digital} \\\cmidrule{4-6} \cmidrule{7-10} \cmidrule{11-14} \cmidrule{15-18}
& & & Gau & shot & Imp & Def & Glass & Mot & Zoom & Snow & Frost & Fog & Bright & Cont & Elas & Pixel & JPEG \\
\midrule
\multirow{4}*{\rotatebox{90}{Res50}}
& Baseline &77.4&78.3&79.5&82.0&73.7&89.7&78.1&81.0&82.6&77.5&68.3&57.8&72.2&85.5&77.4&77.7 \\
& Gaussian &75.7&75.5&76.7&78.3&73.8&88.6&76.3&80.3&79.2&75.7&69.0&57.2&71.6&84.3&74.7&74.8 \\
& AdvProp &69.9&74.0&74.2&76.2&69.4&\bf79.6&74.0&\bf73.8&\bf74.0&\bf67.6&\bf65.0&\bf52.2&68.3&\bf77.1&61.2&\bf61.4 \\
& AdvWavAug &\bf69.0&\bf66.7&\bf68.4&\bf68.8&\bf68.4&81.0&\bf73.5&74.9&78.2&71.1&66.6&53.5&\bf67.9&77.2&\bf56.2&63.3 \\
\midrule
\multirow{4}*{\rotatebox{90}{Res101}}
& Baseline &70.8&71.8&74.0&73.4&68.0&82.8&73.9&77.3&74.8&71.5&63.8&53.0&66.4&79.1&63.4&68.0 \\
& Gaussian &69.4&68.4&69.8&69.7&67.7&82.6&73.7&75.2&74.4&69.8&61.7&52.0&66.4&78.2&62.7&68.0 \\
& AdvProp &65.5&65.4&67.7&67.7&64.1&\bf75.5&\bf67.0&\bf69.7&\bf71.2&\bf65.6&64.9&\bf49.3&66.7&\bf70.7&58.7&\bf58.5 \\
& AdvWavAug &\bf63.7&\bf60.2&\bf61.5&\bf61.8&\bf62.9&75.7&68.2&70.2&72.4&65.9&\bf62.3&49.7&\bf63.2&70.8&\bf50.4&60.0 \\
\midrule
\multirow{4}*{\rotatebox{90}{Res152}}
& Baseline &69.1&69.0&70.9&71.2&65.5&83.3&71.1&74.0&73.6&70.0&61.8&51.6&65.2&77.5&63.1&68.9 \\
& Gaussian &67.3&66.8&68.5&67.8&66.3&81.0&67.5&71.8&72.2&68.9&61.9&49.9&63.8&76.2&60.8&66.2 \\
& AdvProp &\bf62.3&62.2&63.5&63.9&\bf61.9&\bf71.7&\bf64.4&\bf67.7&68.9&\bf62.0&60.9&48.0&64.1&\bf68.7&\bf52.1&\bf55.4 \\
& AdvWavAug &63.2&\bf62.0&\bf63.4&\bf63.6&62.4&74.9&65.8&69.9&\bf68.1&62.8&\bf56.3&\bf47.5&\bf61.9&70.7&60.4&57.8 \\
\midrule
\multirow{4}*{\rotatebox{90}{SwinT}}
& Baseline &67.0&57.1&58.3&57.7&70.9&82.4&69.4&78.1&62.0&58.8&60.5&53.9&50.0&80.1&79.2&86.5 \\
& Gaussian &63.2&52.8&53.8&54.7&69.2&80.9&66.9&77.7&56.3&52.9&54.9&49.8&45.7&78.3&74.9&79.3 \\
& AdvProp &61.9&51.2&52.2&53.0&68.1&79.4&67.3&76.2&54.6&51.8&57.3&49.2&44.0&76.3&74.1&74.6 \\
& AdvWavAug &\bf56.1&\bf46.6&\bf47.4&\bf47.7&\bf62.9&\bf73.6&\bf61.7&\bf71.2&\bf51.6&\bf47.1&\bf53.3&\bf46.1&\bf40.0&\bf70.4&\bf61.2&\bf60.1 \\
\midrule
\multirow{4}*{\rotatebox{90}{SwinS}}
& Baseline &61.1&49.9&51.2&50.8&66.1&78.7&62.4&72.1&57.3&54.1&51.4&49.3&46.8&74.3&69.3&83.4 \\
& Gaussian &58.2&47.8&49.4&48.5&65.0&78.1&62.7&73.2&52.8&51.1&57.9&48.2&44.6&74.5&66.8&52.9 \\
& AdvProp &56.9&45.1&46.6&46.3&63.2&74.9&60.6&71.0&50.6&48.9&52.8&46.5&42.8&71.0&62.2&71.8 \\
& AdvWavAug &\bf48.9&\bf37.9&\bf39.1&\bf38.1&\bf56.6&\bf67.3&\bf54.0&\bf63.3&\bf44.4&\bf42.8&\bf45.0&\bf41.0&\bf36.3&\bf63.4&\bf49.4&\bf55.2 \\
\bottomrule
\end{tabular*}
\end{minipage}
\end{center}
\end{sidewaystable*}

\renewcommand\arraystretch{1.2}
\begin{sidewaystable}\small
\begin{center}
\begin{minipage}{\textheight}
\caption{Comparison of generalization with AdvWavAug and VQ-VAE augmentation on different corruptions in ImageNet-C.}
\label{tab:3.2}
\begin{tabular*}{\textheight}{@{\extracolsep{\fill}}lllccccccccccccccc@{\extracolsep{\fill}}}
\toprule%
& Method & mCE \textcolor{red}{$\downarrow$} & \multicolumn{3}{@{}c@{}}{Noise} & \multicolumn{4}{@{}c@{}}{Blur} & \multicolumn{4}{@{}c@{}}{Weather} & \multicolumn{4}{@{}c@{}}{Digital} \\\cmidrule{4-6} \cmidrule{7-10} \cmidrule{11-14} \cmidrule{15-18}
& & & Gau & shot & Imp & Def & Glass & Mot & Zoom & Snow & Frost & Fog & Bright & Cont & Elas & Pixel & JPEG \\
\midrule
\multirow{2}*{\rotatebox{90}{Res50}}
& VQ-VAE &\bf68.0&68.3&70.7&72.1&68.9&\bf79.1&\bf71.2&\bf73.6&\bf73.9&\bf67.5&\bf62.7&\bf52.8&\bf67.6&\bf76.7&57.9&\bf57.3 \\
& AdvWavAug &69.0&\bf66.7&\bf68.4&68.8&68.4&81.0&73.5&74.9&78.2&71.1&66.6&53.5&67.9&77.2&\bf56.2&63.3 \\
\midrule
\multirow{2}*{\rotatebox{90}{Res101}}
& VQ-VAE &\bf62.4&61.9&64.4&64.6&63.6&\bf72.1&\bf67.7&\bf68.0&\bf70.2&\bf63.2&\bf57.9&\bf48.3&\bf61.7&\bf70.0&50.8&\bf50.8 \\
& AdvWavAug &63.7&\bf60.2&\bf61.5&\bf61.8&\bf62.9&75.7&68.2&70.2&72.4&65.9&62.3&49.7&63.2&70.8&\bf50.4&60.0 \\
\midrule
\multirow{2}*{\rotatebox{90}{Res152}}
& VQ-VAE &\bf58.2&\bf56.0&\bf57.9&\bf57.3&\bf60.2&\bf69.8&\bf63.0&\bf64.1&\bf65.2&\bf58.4&\bf54.4&\bf45.1&\bf57.8&\bf65.8&\bf48.9&\bf49.4 \\
& AdvWavAug &63.2&62.0&63.4&63.6&62.4&74.9&65.8&69.9&68.1&62.8&56.3&47.5&61.9&70.7&60.4&57.8 \\
\midrule
\multirow{2}*{\rotatebox{90}{SwinT}}
& VQ-VAE &57.0&\bf44.6&\bf45.4&\bf45.8&65.8&\bf71.6&64.1&73.3&51.9&48.0&58.9&\bf44.8&42.2&\bf69.7&68.7&60.4 \\
& AdvWavAug &\bf56.1&46.6&47.4&47.7&\bf62.9&73.6&\bf61.7&\bf71.2&\bf51.6&\bf47.1&\bf53.3&46.1&\bf40.0&70.4&\bf61.2&\bf60.1 \\
\midrule
\multirow{2}*{\rotatebox{90}{SwinS}}
& VQ-VAE &49.1&\bf37.8&\bf38.3&\bf38.0&57.7&\bf65.1&54.3&63.8&\bf43.5&\bf42.2&53.1&\bf39.9&37.4&\bf62.5&\bf49.4&\bf53.4 \\
& AdvWavAug &\bf48.9&37.9&39.1&38.1&\bf56.6&67.3&\bf54.0&\bf63.3&44.4&42.8&\bf45.0&41.0&\bf36.3&63.4&\bf49.4&55.2 \\
\bottomrule
\end{tabular*}
\end{minipage}
\end{center}
\end{sidewaystable}

\renewcommand\arraystretch{1.2}
\begin{sidewaystable}\small
\begin{center}
\begin{minipage}{\textheight}
\caption{Comparison of generalization with baseline, AdvWavAug and VQ-VAE augmentation on different corruptions in ImageNet-C.}
\label{tab:4.2}
\begin{tabular*}{\textheight}{@{\extracolsep{\fill}}lllccccccccccccccc@{\extracolsep{\fill}}}
\toprule%
& Method & mCE \textcolor{red}{$\downarrow$} & \multicolumn{3}{@{}c@{}}{Noise} & \multicolumn{4}{@{}c@{}}{Blur} & \multicolumn{4}{@{}c@{}}{Weather} & \multicolumn{4}{@{}c@{}}{Digital} \\\cmidrule{4-6} \cmidrule{7-10} \cmidrule{11-14} \cmidrule{15-18}
& & & Gau & shot & Imp & Def & Glass & Mot & Zoom & Snow & Frost & Fog & Bright & Cont & Elas & Pixel & JPEG \\
\midrule
\multirow{3}*{\rotatebox{90}{SwinT}}
& Baseline &62.0&52.2&53.7&53.6&67.9&78.6&64.1&75.3&55.9&52.8&51.3&48.1&45.1&75.7&76.3&79.1 \\
& VQ-VAE &60.3&48.7&49.8&50.7&66.8&76.9&65.2&75.4&52.4&49.2&50.7&46.8&45.5&75.3&76.5&75.0 \\
& AdvWavAug &\bf53.2&\bf44.3&\bf46.1&\bf46.3&\bf60.8&\bf70.9&\bf60.1&\bf70.0&\bf46.4&\bf43.6&\bf49.0&\bf42.1&\bf38.0&\bf66.2&\bf59.2&\bf54.9 \\
\midrule
\multirow{3}*{\rotatebox{90}{SwinS}}
& Baseline &54.9&42.9&44.9&43.3&61.3&74.1&56.6&67.5&50.9&48.5&46.0&44.1&42.1&68.9&62.1&70.7 \\
& VQ-VAE &46.2&35.6&35.7&35.1&55.9&\bf63.8&53.0&62.2&42.3&39.1&43.6&\bf37.6&33.8&\bf60.5&45.5&\bf49.0 \\
& AdvWavAug &\bf45.8&\bf34.4&\bf35.3&\bf34.9&\bf54.1&64.1&\bf51.4&\bf61.8&\bf40.7&\bf38.9&\bf43.2&37.8&\bf33.4&61.6&\bf44.8&50.6 \\
\midrule
\multirow{3}*{\rotatebox{90}{SwinB}}
& Baseline &54.5&43.4&44.9&43.8&61.4&71.4&55.1&66.7&50.0&48.4&47.2&43.2&38.9&70.4&65.5&66.6 \\
& VQ-VAE &50.2&38.7&39.7&39.6&57.3&66.0&54.4&63.9&47.1&42.9&44.1&38.3&34.6&64.6&62.0&60.1 \\
& AdvWavAug &\bf44.9&\bf33.9&\bf35.2&\bf34.2&\bf53.5&\bf62.8&\bf50.9&\bf61.1&\bf40.1&\bf37.9&\bf39.4&\bf36.8&\bf31.7&\bf60.6&\bf45.9&\bf49.5 \\
\bottomrule
\end{tabular*}
\end{minipage}
\end{center}
\end{sidewaystable}

\renewcommand\arraystretch{1.2}
\begin{sidewaystable}\small
\begin{center}
\begin{minipage}{\textheight}
\caption{Combination with another augmentation technique AugMix on different corruptions in ImageNet-C. We combine our $\rm AdvWavAug$ with AugMix, and compare its performance with the original AugMix on ResNet-50.}
\label{tab:5.2}
\begin{tabular*}{\textheight}{@{\extracolsep{\fill}}lllccccccccccccccc@{\extracolsep{\fill}}}
\toprule%
& Method & mCE \textcolor{red}{$\downarrow$} & \multicolumn{3}{@{}c@{}}{Noise} & \multicolumn{4}{@{}c@{}}{Blur} & \multicolumn{4}{@{}c@{}}{Weather} & \multicolumn{4}{@{}c@{}}{Digital} \\\cmidrule{4-6} \cmidrule{7-10} \cmidrule{11-14} \cmidrule{15-18}
& & & Gau & shot & Imp & Def & Glass & Mot & Zoom & Snow & Frost & Fog & Bright & Cont & Elas & Pixel & JPEG \\
\midrule
\multirow{4}*{\rotatebox{90}{Res50}}
& Baseline &77.4&78.3&79.5&82.0&73.7&89.7&78.1&81.0&82.6&77.5&68.3&57.8&72.2&85.5&77.4&77.7 \\
\cmidrule{2-18}
& AugMix &65.3&67.0&66.0&68.0&64.0&79.0&\bf59.0&64.0&\bf69.0&68.0&65.0&54.0&\bf57.0&\bf74.0&60.0&66.0 \\
\cmidrule{2-18}
& $\rm AdvWavAug$ &68.0&65.2&67.1&66.3&67.5&77.5&72.3&71.9&75.5&68.9&67.4&53.2&68.5&75.7&\bf59.8&62.5 \\
& +AugMix &\bf63.3&\bf61.3&\bf60.4&\bf59.8&\bf61.2&\bf74.3&59.1&\bf61.9&71.8&\bf66.7&\bf60.7&\bf51.9&62.6&74.7&61.3&\bf61.4 \\
\bottomrule
\end{tabular*}
\end{minipage}
\end{center}
\end{sidewaystable}

\renewcommand\arraystretch{1.2}
\begin{sidewaystable}\small
\begin{center}
\begin{minipage}{\textheight}
\caption{Comparison of generalization with MAE and MAE+AdvWavAug on different corruptions in ImageNet-C.}
\label{tab:6.2}
\begin{tabular*}{\textheight}{@{\extracolsep{\fill}}lllccccccccccccccc@{\extracolsep{\fill}}}
\toprule%
& Method & mCE \textcolor{red}{$\downarrow$} & \multicolumn{3}{@{}c@{}}{Noise} & \multicolumn{4}{@{}c@{}}{Blur} & \multicolumn{4}{@{}c@{}}{Weather} & \multicolumn{4}{@{}c@{}}{Digital} \\\cmidrule{4-6} \cmidrule{7-10} \cmidrule{11-14} \cmidrule{15-18}
& & & Gau & shot & Imp & Def & Glass & Mot & Zoom & Snow & Frost & Fog & Bright & Cont & Elas & Pixel & JPEG \\
\midrule
\multirow{2}*{\rotatebox{90}{ViTB}}
& Baseline &51.7&-&-&-&-&-&-&-&-&-&-&-&-&-&-&- \\
& AdvWavAug &\bf47.0&34.9&35.6&35.3&57.7&65.6&51.8&63.8&38.8&38.5&39.5&38.1&34.6&66.7&48.4&55.0 \\
\midrule
\multirow{2}*{\rotatebox{90}{ViTL}}
& Baseline &41.8&-&-&-&-&-&-&-&-&-&-&-&-&-&-&- \\
& AdvWavAug &\bf38.1&27.2&27.2&26.4&48.7&56.8&40.4&49.7&30.1&31.3&32.4&32.5&29.2&55.5&38.8&45.4 \\
\midrule
\multirow{2}*{\rotatebox{90}{ViTH}}
& Baseline &33.8&-&-&-&-&-&-&-&-&-&-&-&-&-&-&- \\
& AdvWavAug &\bf32.6&23.9&22.7&23.8&40.6&50.2&32.6&39.9&24.9&29.4&25.8&28.2&27.9&50.0&30.7&38.0 \\
\bottomrule
\end{tabular*}
The baseline models are taken directly from~\citep{he2021masked}, which have no detailed results on different corruptions.
\end{minipage}
\end{center}
\end{sidewaystable}

\renewcommand\arraystretch{1.2}
\begin{sidewaystable}\small
\begin{center}
\begin{minipage}{\textheight}
\caption{Comparison of generalization with AdvWavAug and AdvWavAug+PGD on different corruptions in ImageNet-C.}
\label{tab:7.2}
\begin{tabular*}{\textheight}{@{\extracolsep{\fill}}lllccccccccccccccc@{\extracolsep{\fill}}}
\toprule%
& Method & mCE \textcolor{red}{$\downarrow$} & \multicolumn{3}{@{}c@{}}{Noise} & \multicolumn{4}{@{}c@{}}{Blur} & \multicolumn{4}{@{}c@{}}{Weather} & \multicolumn{4}{@{}c@{}}{Digital} \\\cmidrule{4-6} \cmidrule{7-10} \cmidrule{11-14} \cmidrule{15-18}
& & & Gau & shot & Imp & Def & Glass & Mot & Zoom & Snow & Frost & Fog & Bright & Cont & Elas & Pixel & JPEG \\
\midrule
\multirow{2}*{\rotatebox{90}{Res50}}
& AdvWavAug &69.0&66.7&68.4&68.8&68.4&81.0&73.5&74.9&78.2&71.1&\bf66.6&53.5&\bf67.9&77.2&\bf56.2&63.3 \\
& +PGD &\bf68.0&\bf65.2&\bf67.1&\bf66.3&\bf67.5&\bf77.5&\bf72.3&\bf71.9&\bf75.5&\bf68.9&67.4&\bf53.2&68.5&\bf75.7&59.8&\bf62.5 \\
\midrule
\multirow{2}*{\rotatebox{90}{Res101}}
& AdvWavAug &63.7&\bf60.2&\bf61.5&\bf61.8&\bf62.9&75.7&68.2&70.2&72.4&65.9&\bf62.3&49.7&\bf63.2&\bf70.8&50.4&60.0 \\
& +PGD &\bf63.0&61.0&62.9&62.5&64.0&\bf73.1&\bf65.6&\bf68.4&\bf71.6&\bf64.5&62.8&\bf48.7&63.9&71.0&\bf49.4&\bf55.7 \\
\midrule
\multirow{2}*{\rotatebox{90}{Res152}}
& AdvWavAug &63.2&62.0&63.4&63.6&\bf62.4&74.9&\bf65.8&69.9&\bf68.1&\bf62.8&\bf56.3&\bf47.5&\bf61.9&70.7&60.4&\bf57.8 \\
& +PGD &\bf62.7&\bf60.4&\bf61.6&\bf62.4&62.9&\bf70.7&65.7&\bf67.4&71.0&63.8&64.2&50.6&65.7&\bf68.1&\bf48.1&58.3 \\
\midrule
\multirow{2}*{\rotatebox{90}{SwinT}}
& AdvWavAug &\bf56.1&46.6&\bf47.4&47.7&\bf62.9&\bf73.6&\bf61.7&\bf71.2&51.6&\bf47.1&53.3&\bf46.1&\bf40.0&\bf70.4&\bf61.2&\bf60.1 \\
& +PGD &56.6&\bf46.4&\bf47.4&\bf47.4&65.0&74.6&62.0&71.9&\bf50.9&47.4&\bf52.7&46.3&40.3&70.8&63.6&62.5 \\
\midrule
\multirow{2}*{\rotatebox{90}{SwinS}}
& AdvWavAug &\bf48.9&\bf37.9&\bf39.1&\bf38.1&\bf56.6&\bf67.3&\bf54.0&\bf63.3&\bf44.4&\bf42.8&\bf45.0&\bf41.0&\bf36.3&\bf63.4&\bf49.4&\bf55.2 \\
& +PGD &53.4&41.0&42.2&42.2&60.4&72.0&58.3&68.9&47.7&45.0&50.2&44.3&39.7&66.6&61.7&61.4 \\
\bottomrule
\end{tabular*}
\end{minipage}
\end{center}
\end{sidewaystable}

\subsection{Experimental Results on Black-box Attacks}
\label{sec:c.5}

In addition to white-box attacks, black-box attacks are also important for evaluating the attack performance of a method. perspective to measure the attack performance of a method. We conduct black-box attacks on various models, including  ResNet 50 (Res50)~(\cite{he2016deep}), ResNet 152 (Res152)~(\cite{he2016deep}), Inception V3 (IncV3)~(\cite{szegedy2016rethinking}), Adv Inception V3 (AdvIncV3)~(\cite{kurakin2018adversarial}),  Efficientnet B6 AP (EffiAP)~(\cite{tan2019efficientnet}), and Swin Transformer Tiny (SwinT)~\citep{liu2021swin}. We set the attack iteration number as $n=100$. The ASR-FID curves are shown in Fig.~\ref{fig:app1}. The ASR-LPIPS curves are shown in Fig.~\ref{fig:app2}. The curves demonstrate that our AdvWavAug-DIM consistently outperforms DIM. The black-box attacks showcase the transferability of our method, indicating that adversarial examples generated by AdvWavAug-DIM exhibit superior performance (better attack performance given the same image quality) compared to adversarial examples generated by DIM. This improvement can be attributed to our method's ability to project the original perturbations onto the data manifold, enhancing the attack effectiveness.

\section{Detailed results on Different Corruptions of ImageNet-C dataset}
\label{appendix:d}
In this section, we provide detailed results on different corruptions (\eg, gaussian noise, glass blur \etc.) from the ImageNet-C dataset. There are too many corruption categories in ImageNet-C, we exhibit the abbreviations of different categories in these tables with Gaussian noise denoted as Gau, shot noise denoted as Shot, impulse noise denoted as Imp, defocus blur denoted as Def, glass blur denoted as Glass, motion blur denoted as Mot, zoom blur denoted as Zoom, snow denoted as Snow, frost denoted as Frost, fog denoted as Fog, brightness denoted as Bright, contrast denoted as Cont, elastic transform denoted as Elas, pixelate denoted as Pixel, JPEG compression denoted as JPEG. We extend the detailed ImageNet-C results in Tab.~2 to Tab.~\ref{tab:2.2}, results in Tab.~3 to Tab.~\ref{tab:3.2}, results in Tab.~4 to Tab.~\ref{tab:4.2}, results in Tab.~5 to Tab.~\ref{tab:5.2}, results in Tab.~6 to Tab.~\ref{tab:6.2}, results in Tab.~7 to Tab.~\ref{tab:7.2}. What should be noticed in Tab.~\ref{tab:6.2} is that we take the results directly from~\citep{he2021masked}, which have no detailed data of different corruptions.
\end{appendix}
\bibliographystyle{elsarticle-harv} 
\bibliography{example}






\end{document}